\documentclass[lettersize,journal]{IEEEtran}
\usepackage{amsmath,amsfonts}
\usepackage{algorithmic}
\usepackage{algorithm}
\usepackage{array}
\usepackage[caption=false,font=normalsize,labelfont=sf,textfont=sf]{subfig}
\usepackage{textcomp}
\usepackage{xcolor}
\usepackage{stfloats}
\usepackage{url}
\usepackage{verbatim}
\usepackage{graphicx}
\usepackage{cite}
\usepackage{booktabs}
\newtheorem{definition}{Definition}
\newtheorem{assumption}{Assumption}
\newtheorem{theorem}{Theorem}
\newtheorem{lemma}{Lemma}
\newtheorem{proof}{Proof}
\begin{document}

\title{Silent Neuron Theory and Plasticity Preservation for Deep Reinforcement Learning in Adaptive Video Streaming}

\author{Zhiqiang He, \IEEEmembership{Member,~IEEE}, Zhi Liu, \IEEEmembership{Senior Member,~IEEE}

\thanks{Zhiqiang He and Zhi Liu are with the Department of Computer and Network Engineering, the University of Electro-Communications, Japan. Email: hezhiqiang@ieee.org, liu@ieee.org.  The work of Zhiqiang was supported by JST SPRING under Grant JPMJSP2131. The work of Zhi Liu was supported in part by JSPS KAKENHI Grant 24K14927 and JST ASPIRE JPMJAP2432.}

\thanks{The corresponding author is Zhi Liu.}
}

\markboth{IEEE Transactions on Mobile Computing Accepted}%
{Shell \MakeLowercase{\textit{et al.}}: A Sample Article Using IEEEtran.cls for IEEE Journals}


\maketitle

\begin{abstract}
Adaptive video streaming optimizes Quality of Experience (QoE) metrics by selecting appropriate bitrates according to varying network bandwidth and user demands. In practice, however, real-world network bandwidth often exhibits heterogeneity relative to training environments. Current methods predominantly tackle this problem through learning-based approaches designed to improve generalization performance. Our systematic investigation reveals a critical limitation: neural networks suffer from plasticity loss, significantly impeding their ability to adapt to heterogeneous network conditions. Through theoretical analysis of neural propagation mechanisms, we demonstrate that existing dormant neuron metrics inadequately characterize neural plasticity loss. To address this limitation, we have developed the Silent Neuron theory, which provides a more comprehensive framework for understanding plasticity degradation. Based on these theoretical insights, we propose the Reset Silent Neuron (ReSiN), which preserves neural plasticity through strategic neuron resets guided by both forward and backward propagation states. Moreover, we establish a tighter performance bound for ReSiN under non-stationary network conditions. In our implementation of an adaptive video streaming system, ReSiN shows substantial gains under non-stationary conditions in which the available bandwidth alternates between high and low regimes, improving the QoE by $178\%$, and it remains beneficial under stationary real-world traces, where it improves the QoE by $1.6\%$. Code will be available: https://tinyzqh.github.io/ReSiN/
\end{abstract}


\section{Introduction}
\IEEEPARstart{A}daptive Bitrate (ABR) streaming optimizes users’ Quality of Experience (QoE) by selecting optimal bitrates under constrained network bandwidth, and has been widely deployed in platforms such as YouTube and Netflix. Existing ABR algorithms can be broadly categorized into learning-based and non-learning-based approaches \cite{peroni2025end}. Learning-based methods adopt a data-driven paradigm that does not rely heavily on prior domain knowledge, and typically outperform heuristics on training traces. However, because their policies are tightly coupled to the training data distribution, they are prone to severe degradation under unseen network conditions—so distribution shift in bandwidth largely determines their performance ceiling \cite{zhang2025novel}. A large body of prior work assumes a stationary bandwidth, where throughput follows a fixed distribution or varies within predictable bounds \cite{li2024learning,wang2024mmvs}.

While this assumption enables tractable optimization and analytical solutions, it overlooks the non-stationarity inherent in real-world networks \cite{cao2001nonstationarity}. In practice, network bandwidth distributions shift dramatically due to factors such as 4G-WiFi handoffs, urban congestion, and mobility pattern changes. The mismatch between the distribution of training traces and that of real-world network conditions encountered during inference induces non-stationarity, which undermines the generalization ability of learning-based methods. The inherent unpredictability and diversity of network bandwidth fluctuations introduce a subtle form of non-stationarity. This fundamental discrepancy between training assumptions and operational conditions creates a significant generalization gap that undermines the effectiveness of learning-based ABR algorithms \cite{luo2025sabr,hoffman2025into}.

Current solutions for non-stationarity in adaptive bitrate streaming mainly focus on introducing an additional classifier to predict such changes, such as preference predictors~\cite{wu2024mansy}, network-bandwidth classifiers~\cite{zhang2025novel,11016689}, or encoding prior knowledge of network bandwidth into latent features via an auto-encoder to assist subsequent decision-making~\cite{kan2025merina+}. However, these approaches first rely on strong prior knowledge and second, shift the generalization problem from the policy network to the classifier \cite{bentaleb2024bitrate,wang2024mmvs,plume2024}. Once the external environment falls outside the classifier’s prediction range, the entire pipeline is at risk of failure, and the non-stationarity problem in ABR remains unresolved at a fundamental level.

In non-stationary ABR systems, where network dynamics are unknown and change over time, traditional rule-based and theoretically derived approaches become infeasible. Therefore, we have adopted a data-driven, learning-based approach that captures fluctuations in internal neuron signals caused by external bandwidth changes and utilizes these fluctuations as a control signal. Notably, PA-MoE \cite{he2025plasticity} follows a similar philosophy, addressing non-stationary ABR by monitoring neural plasticity of the learning agent from within the learning-based framework. It monitors learning capacity via neuron dormancy indicators \cite{dohare2024loss} and preserves plasticity through noise perturbations. Analogous to biological neural systems \cite{puderbaugh2023neuroplasticity}, artificial neural networks exhibit degraded plasticity in non-stationary environments, where inactive neurons fail to adapt to environmental fluctuations \cite{dohare2024loss}. By maintaining plasticity, PA-MoE enables the agent to sustain rapid adaptability in non-stationary ABR environments. 

PA-MoE focuses on non-stationarity from the QoE objective itself, while we focus on non-stationarity driven by network bandwidth dynamics. PA-MoE achieves good performance by adding noise, but analyzes the network as a whole rather than examining individual neurons, leaving the root causes of plasticity loss insufficiently analyzed \cite{LyleZNPPD23, gulcehre2022an, Hare_Tortoise}. This paper provides a deeper investigation by analyzing each neuron individually, enabling precise identification and reset of dormant neurons. Existing work typically identifies neurons with zero output as dormant \cite{sokar2023dormant}, but we argue that this forward-propagation-only criterion is incomplete. We therefore propose a novel dual-zero criterion that considers both forward propagation and gradient values: a neuron is classified as silent only when both are zero, indicating complete non-participation in information flow and learning. This distinction is crucial, as a neuron with zero forward output may still actively learn via non-zero gradients, and vice versa.

This rigorous characterization leads to a novel theoretical framework for analyzing neural plasticity loss through the lens of forward and backward propagation in optimization. We formalize the concept of Silent Neurons - neurons exhibiting simultaneous zero forward propagation and zero gradients - and establish their relationship to plasticity degradation in non-stationary environments. Based on this theoretical foundation, we develop Reset Silent Neuron (ReSiN), a method that systematically identifies and reactivates Silent Neurons to maintain neural plasticity. Through extensive experiments analyzing neural network dynamics, we validate both our theoretical analysis and the effectiveness of ReSiN. The results demonstrate that ReSiN not only preserves neural plasticity but also achieves significant performance improvements in non-stationary network resource adaptation scenarios.

Our key contributions are:

\begin{itemize}
    \item \textbf{Problem Analysis:} We provide, to the best of our knowledge, the first systematic analysis of how neural plasticity affects adaptive video streaming under non-stationary bandwidth. Our study shows that internal neural dynamics directly shape the agent’s ability to adapt, thereby deepening the understanding of plasticity in ABR algorithms.
    
    \item \textbf{Theoretical Foundation:} We establish a theoretical framework that links forward and backward propagation to plasticity loss. We formally define Silent Neurons and prove that both types of signals must be considered jointly to accurately characterize conditions under which plasticity is lost. We further derive a tighter performance bound in non-stationary environments.

    \item \textbf{Algorithm Design:} Based on our theoretical insights, we have developed Reset Silent Neuron (ReSiN), a practical method for detecting and reactivating Silent Neurons. This approach allows for continuous learning without the need for prior knowledge of network conditions or environmental characteristics.

    \item \textbf{Implementation and Evaluation:} We implement ReSiN in an adaptive video streaming system and evaluate it on real-world network traces. Using standardized metrics, fixed hyperparameters, and consistent random seeds, we demonstrate that ReSiN yields substantial QoE gains over strong baselines while ensuring fair and reproducible comparisons.
\end{itemize}

The rest of this paper is organized as follows. Section \ref{background_and_motivation} motivates our work and identifies the research problems. Section \ref{StochasticNetworkResourceAdaptationProblem} formulates the system model for stochastic adaptive video streaming. Section \ref{silentneuron} presents our theoretical analysis of Silent Neurons and proposes the ReSiN method. Section \ref{experiment} provides comprehensive experimental evaluation. Section \ref{related_work} discusses related work and contextualizes our contributions. Section \ref{conclusion} concludes with future research directions.

\section{Motivation and Analysis}
\label{background_and_motivation}

In this section, we motivate our approach with a concrete example of ABR under non-stationary bandwidth, analyze the underlying causes from a neural plasticity perspective, and provide an intuitive explanation of our solution.

\subsection{Impact of Non-stationary Bandwidth on ABR}

Given that Proximal Policy Optimization (PPO) \cite{schulman2017proximal} has been widely adopted as the reinforcement learning algorithm of choice in various state-of-the-art AVS approaches \cite{11016689, luo2025sabr, he2025plasticity}, we likewise employ PPO in this work and conduct our subsequent analyses based on it. To investigate the impact of network bandwidth on performance, we utilize a total of 127 training traces drawn from HSDPA-based 3G networks \cite{riiser2013commute}, the FCC corpus \cite{kan2022improving}, and the Puffer open dataset \cite{yan2020learning}. We adopt the bandwidth categorization scheme from \cite{zhang2025novel}, classifying traces with an average bandwidth below 2 Mbps as low bandwidth (LBW) and those above 3 Mbps as high bandwidth (HBW). These categorizations enable us to examine how PPO behaves under nonstationary network conditions, which we simulate as abrupt transitions by switching between high and low network bandwidth.

\begin{figure}[!htbp]  
  \centering
  \includegraphics[width=\columnwidth]{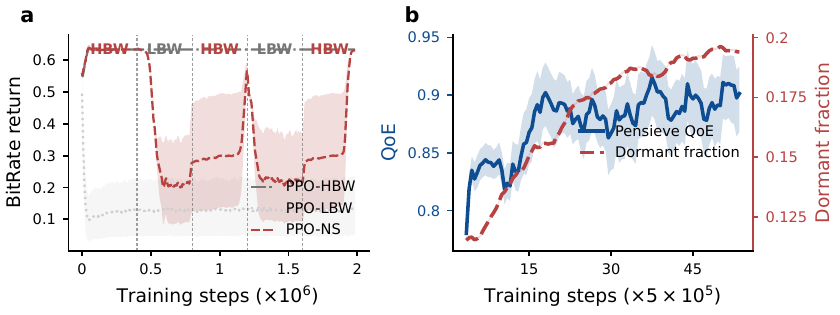}
  \caption{A case study of PPO's performance in adaptive video streaming under different bandwidth conditions.}
  \label{ppo_bitrate_comparison}
\end{figure}

Results in Fig. \ref{ppo_bitrate_comparison} (a) reveal a critical limitation of PPO Agent under non-stationary conditions. When transitioning from HBW to LBW, the agent fails to match the performance baseline established in consistent LBW environments. More significantly, upon returning to HBW conditions, the agent cannot recover its original performance level, despite having previously optimized for this exact environment. Note that the HBW/LBW switching described here is an exaggerated scenario intended to highlight the limitations of the model. However, similar abrupt transitions can occur in practical situations, such as when a model trained in one environment is deployed in another, or when a user's network condition switches from 3G to WiFi. Another form of non-stationarity, which is more gradual, can also occur within a single bandwidth regime. In this case, the same adaptation issue is expected to persist, although it may be less noticeable. This is demonstrated in Fig. \ref{ppo_bitrate_comparison} (b), which shows the evolution of QoE and dormant neurons for Pensieve \cite{mao2017neural} under real-world traces. Even in this scenario, the number of neurons outputting zero increases over time, a phenomenon known as loss of plasticity \cite{sokar2023dormant}. Therefore, we first analyze abrupt non-stationarity as an amplified and more tractable case of this phenomenon, and then apply the resulting algorithm to the subtler non-stationarity observed in real-world traces.

\subsection{Neural Plasticity Analysis for Non-stationary Bandwidth}

In this section, we analyze—through the lens of network plasticity—why PPO tends to lose its learning capability under continuously shifting, non-stationary bandwidth conditions. Prior work has predominantly characterized such plasticity degradation by examining dormant neurons, i.e., neurons whose outputs are identically zero during the forward pass \cite{sokar2023dormant, xu2024drm}, as formally defined in Definition~\ref{dormant_neuron}. Building upon this dormancy metric, we assess the evolution of PPO’s plasticity throughout training. In addition, we further evaluate whether neurons, once entering a dormant state for the first time, remain persistently inactive thereafter and fail to recover, as defined in Definition~\ref{OverlapCoefficientforNeuron}.

\begin{figure}[!htbp]
    \centering
    \includegraphics[width=\columnwidth]{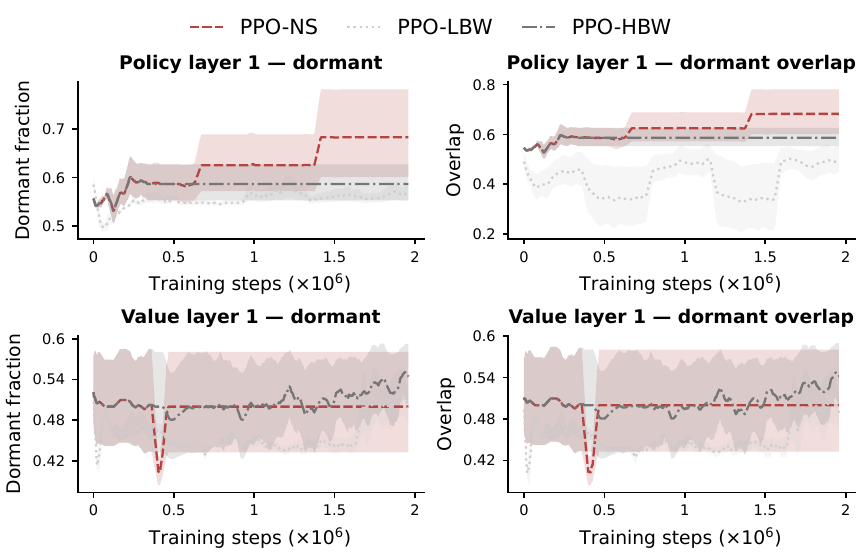}%
    \caption{Evolution of neural network plasticity loss.}
    \label{ppo_dormant_overleap}
\end{figure}

Fig. \ref{ppo_dormant_overleap} (left) shows that a substantial proportion of neurons in both the policy and value networks become dormant, and the number of dormant neurons continues to increase under non-stationary conditions. The right panel further illustrates the fraction of neurons that, once entering a dormant state, remain inactive throughout all subsequent steps. Notably, the proportions and trends in the right panel closely mirror those observed on the left. This suggests that, in non-stationary ABR environments, neurons that become dormant are largely unable to recover, indicating a systematic and persistent degradation of neural plasticity. Once a large portion of neurons lose their plasticity, the network is left with insufficient active capacity to adapt to future environmental changes.

Fig. 2 illustrates the proportion of neurons whose forward-pass activations are exactly zero, which we treat as an indicator of neuronal dormancy. However, we argue that this characterization is incomplete, as it overlooks the backward pass dynamics that are essential for learning and adaptation. Our analysis reveals that true plasticity loss occurs only when neurons exhibit both zero outputs in the forward pass and zero gradients in the backward pass, a condition we term as “Silent Neurons”. This dual-zero state indicates a complete disconnection from both information propagation and learning processes.

\section{System Model}
\label{StochasticNetworkResourceAdaptationProblem}

In this section, we introduce the notation and system model used in our adaptive video streaming framework.

\subsection{Stochastic Adaptive Video Streaming}

Following the adaptive video streaming model \cite{yin2015control}, a video is divided into $I$ consecutive chunks $\{V_1, V_2, \dots, V_I\}$, each with a duration of $t_{i}$-seconds. Each chunk is encoded at multiple bitrate levels $b \in \mathbb{B}$, where $\mathbb{B} = \{b_1, b_2, ..., b_M\}$ represents the available bitrates in ascending order. The size of chunk $i$ encoded at bitrate $b$ is denoted as $d_i(V_i^b)$, which increases monotonically with $b$. The task of the reinforcement learning agent is to select the optimal bitrate for each chunk in order to maximize the quality of the streaming experience. To prevent interruptions during bitrate transitions between segments $V_i$ and $V_{i+1}$, we implement a playback buffer $B_f$. This buffer pre-caches upcoming segments, allowing for concurrent video playback and segment downloading. The buffer occupancy $B_f(t)$ is limited to the range of $[0, B_f^{max}]$, where $B_f^{max}$ is determined by server policies or client storage constraints. However, network bandwidth is non-stationary and exhibits shifting mean and variance, which presents significant challenges for adaptive streaming. Our optimization objective takes into account three key factors: maximizing video quality by selecting higher bitrates during stable periods, minimizing playback interruptions, and maintaining quality stability by avoiding frequent bitrate switches.

The download time for segment $V_i^b$ consists of the network round-trip time $t_{RTT}$ and the data transfer time $d_i(V_i^b)/S_i$, where $S_i$ represents the download speed. To reflect network variability, we incorporate a multiplicative noise factor in the relationship between consecutive segment start download times in $t_{i+1}^{d} = t_{i}^{d} + \left( \frac{d_{i}(V_{i}^{b})}{S_{i}} + t_{RTT} \right) \times \eta(i) + \Delta t_{i}^{w}$, where $\eta(i)$ captures network variability and $\Delta t_i^w$ represents intentional waiting time for buffer management. We define the total network delay as $t_{delay} = \left( \frac{d_{i}(V_{i}^{b})}{S_{i}} + t_{RTT} \right) \times \eta(i)$.

The effective throughput, denoted as $C_{t_{i}}$, is calculated by taking the average download speed and excluding any waiting periods: $C_{t_{i}} = \frac{1}{t_{i+1}^d - t_i^d - \Delta t_i^w} \int_{t_i^d}^{t_{i+1}^d - \Delta t_i^w} S_t dt$. Buffer occupancy evolution captures the interplay between content downloading and playback. When a chunk of content, denoted as $V_{i}$, is successfully downloaded, it adds $t_{i}$ seconds to the buffer. This results in the following dynamics: $B_{f}(t_{i+1}) = \left( \left(B_{f}\left(t_{i} \right) - t_{\text{delay}} \right)_{+} + t_{i} - \Delta t_{i}^{w}  \right)_{+}$, where $(x)_{+}$ represents the non-negative operator $\max(0,x)$. This formulation addresses two critical buffer states: underflow and overflow. Buffer underflow occurs when $B_f(t_i) < t_{\text{delay}}$, which can result in playback interruption. To prevent buffer overflow, the player implements a waiting period: $\Delta t_{i}^{w} = \left( \left(B_{f}\left(t_{i} \right) - t_{\text{delay}} \right)_{+} + t_{i} - B_{f}^{max}  \right)_{+}$. This mechanism maintains buffer occupancy below $B_f^{max}$ while ensuring efficient resource utilization and smooth playback.

\subsection{Objective Function}

The QoE objective is achieved by integrating key factors affecting user satisfaction. The average video quality is measured as $\frac{1}{I} \sum_{i=1}^{I} q(V_i^{b})$, where $q(\cdot)$ is a monotonically non-decreasing quality function. To capture perceptual sensitivity to quality fluctuations, we include a variation penalty $\frac{1}{I-1} \sum_{i=1}^{I-1} |q(V_{i+1}^{b}) - q(V_i^{b})|$. We further account for playback stalls caused by buffer depletion, which are known to heavily degrade user experience. These components are combined in a weighted-sum QoE formulation, allowing for systematic optimization while balancing the trade-offs among quality, stability, and stall avoidance: $QoE_{1}^{I} = \sum_{i=1}^{I} q(V_{i}^{b}) - \mu_{1} \sum_{i=1}^{I-1}|q(V_{i + 1}^{b}) - q(V_{i}^{b})| - \mu_{2} \! \sum_{i=1}^{I} \! \left( \! \left( \! \frac{d_{i}(V_{i}^{b})}{S_{i}} \! +  \!t_{RTT} \! \right) \! \times \! \eta(i) \! - \! B_{f}(t_{i}) \! \right)_{+}.$

The composite QoE metric uses two weighting coefficients, $\mu_1$ and $\mu_2$, to balance the competing objectives of maximizing average quality, stabilizing quality fluctuations, and reducing playback stalls. The system parameter follows \cite{kan2022improving}. This multi-objective formulation enables systematic navigation of the QoE trade-off space, allowing streaming algorithms to adjust their behavior to different deployment requirements and user preferences. To capture the non-linear characteristics of human visual perception, we follow \cite{kan2022improving} and adopt the quality function $\log(\alpha + b_{i})$, where $b_{i}$ denotes the bitrate of chunk $i$, and $\alpha$ is a constant. To account for resource constraints and potential performance degradation at excessively high bitrates, we introduce a penalty term $-\frac{\beta}{b_i}$, with $\beta$ controlling its magnitude. This yields the following quality formulation: $q(V_{i}^{b}) = \log(\alpha + b_{i}) - \frac{\beta}{b_{i}}$. In summary, the problem can be expressed as follows:

\begin{align}
    \begin{split}
        & \max_{V_{1}, \cdots, V_{I}, t_{s}} QoE_{1}^{I} \\
        st. & \quad  t_{i+1}^{d} = t_{i}^{d} + \left( \frac{d_{i}(V_{i}^{b})}{S_{i}} + t_{RTT} \right) \times \eta(i) + \Delta t_{i}^{w}; \\
        & \quad C_t = \frac{1}{t_{i+1}^d - t_i^d - \Delta t_i^w} \int_{t_i^d}^{t_{i+1}^d - \Delta t_i^w} S_t dt; \\
        & \quad B_{f}(t_{i+1}) = \left( \left(B_{f}\left(t_{i} \right) - t_{\text{delay}} \right)_{+} + t_{i} - \Delta t_{i}^{w}  \right)_{+}; \\
        & \quad q(V_{i}^{b}) = \log(\alpha + b_{i}) - \frac{\beta}{b_{i}}; \quad B_{f}(t_{i}) \in [0, B_f^{max}]; \\
        & \quad b_{i} \in \mathbb{B}, \quad \forall i = 1, \cdots I ; \quad \eta(i) \sim \mathcal{U}(0.9,1.1).
    \end{split}
\end{align}

\section{Methodology}
\label{silentneuron}

In this subsection, we present the design and implementation of the proposed ReSiN method, including the formulation of the problem as a Markov Decision Process (MDP) for subsequent solution, the theoretical intuition underlying ReSiN, and the detailed algorithmic design and theoretical analysis.

\subsection{Preliminaries}

RL offers a promising approach to adapting network resources by utilizing Markov Decision Processes (MDPs) to model uncertainties. In an MDP, network scenarios and adaptation decisions are represented as states and actions, respectively. Formally, an MDP is defined by a tuple $(\mathcal{X}, \mathcal{A}, \mathcal{P}, \mathcal{R}, \Gamma)$. The state space $\mathcal{X}$ encompasses all possible network situations, while the action space $\mathcal{A}$ includes all available adaptation decisions. The state transition function $\mathcal{P}: \mathcal{X} \times \mathcal{A} \rightarrow \Delta(\mathcal{X})$ describes the dynamics of the environment, where $\Delta(\mathcal{X})$ denotes the space of probability distributions over $\mathcal{X}$. Specifically, at time step $t$, given the current state $x_t$ and action $a_t$, the probability of transitioning to the next state $x_{t+1}$ is defined as $p(x_{t+1} \mid x_t, a_t)$. The reward function $\mathcal{R}: \mathcal{X} \times \mathcal{A} \rightarrow \mathbb{R}$ provides immediate feedback $r(x_t, a_t)$, while the discount factor $\Gamma \in [0,1)$ balances immediate and future rewards.

While this MDP formulation naturally captures bandwidth variations within a single network through transition probabilities $p(x_{t+1} \mid x_t, a_t)$, it fails to address the fundamental non-stationarity in modern networks. Common scenarios such as switching between 4G and WiFi, transitioning from cellular to tunnel environments, or roaming between carriers introduce structural changes in network characteristics. These changes manifest as shifts in both transition dynamics $\mathcal{P}$ and reward structure $\mathcal{R}$, effectively creating transitions between different MDPs $\{M_1, M_2, ..., M_N\}$, where each $M_{n} = (\mathcal{X}, \mathcal{A}, \mathcal{P}_n, \mathcal{R}_n, \Gamma)$ represents a distinct network environment. This non-stationarity fundamentally challenges learning algorithms, as policies optimized for one network environment can fail dramatically in another. Therefore, in non-stationary environments where the underlying MDP changes rapidly, effective adaptation requires sufficient neuronal plasticity. When a large portion of neurons becomes dormant and loses plasticity, the network can no longer adapt to these environmental shifts.

\subsection{State, Action, and Reward Function Specification}

Our state space, action space, and reward function are consistent with those adopted in existing learning-based approaches, including Pensieve \cite{mao2017neural}, Merina \cite{kan2025merina+}, and PA-MoE \cite{he2025plasticity}.

\textbf{State Space:} The input state at time step $t$ is represented as $x_{t} \in \mathbb{R}^{6 \times 8}$, which encodes a temporal window of 8 consecutive time steps with 6-dimensional features at each step. The 6-dimensional state vector comprises the previous bitrate selection $b_{t-1}$, buffer occupancy $B_{f}(t)$, network throughput, network delay $t_{delay}$, chunk sizes across all bitrate options, and the remaining number of chunks. Formally, $x_{t} = [x_{t-7}, x_{t-6}, \cdots , x_{t}] \in \mathbb{R}^{6 \times 8}$, where each $x_{t}$ is a 6-dimensional vector. This temporal encoding enables the agent to capture both instantaneous network conditions and historical trends for robust decision-making.

\textbf{Action Space:} The action space consists of six discrete bitrate levels $\mathcal{A}=\{300, 750, 1200, 1850, 2850, 4300\}$ $\mathrm{kbps}$, where the agent selects an action $a_{t} \in \{0, 1, 2, 3, 4,5\}$ to determine the quality representation of each downloaded chunk.

\textbf{Reward Function:} The reward function is designed to maximize QoE by balancing video quality, quality variation, and playback stalls.

\subsection{Intuition Behind Our ReSiN Method}

We demonstrate the optimization objective of our method through a simple example. Fig.~\ref{ppo_dormant_overleap} shows that neurons suffer from plasticity loss in non-stationary environments. Beyond environmental bandwidth shifts, the optimization objective itself evolves: early training prioritizes exploration, while later stages emphasize performance improvement. This shifting objective causes available neurons to dwindle over time, progressively constraining the solution space explored by the neural network (purple circle, Fig.~\ref{intuitive_illustration}(b)). Consequently, when the environment changes, insufficient parameters remain to adapt. PA-MoE addresses this by preserving neuron plasticity through noise injection, maintaining a larger explorable solution space (blue circle, Fig.~\ref{intuitive_illustration}(b)) and enabling effective adaptation to environmental non-stationarity.

\begin{figure*}[!htbp]
\centering
\includegraphics[width=\textwidth]{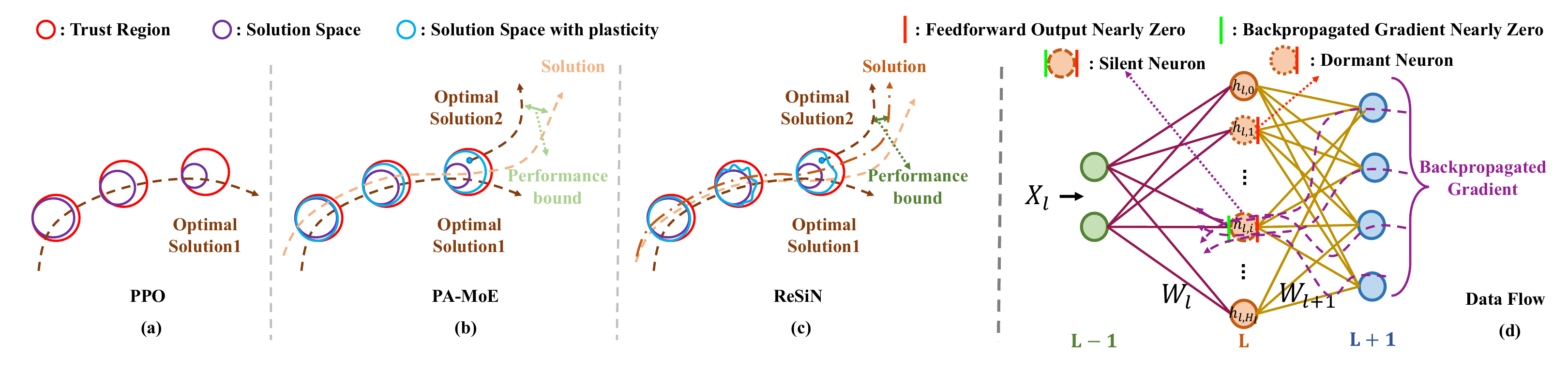}%
\caption{Intuitive Illustration of the Theoretical Improvement. ReSiN provides a tighter performance bound.}
\label{intuitive_illustration}
\end{figure*}

However, not every neuron needs to be injected with noise. Some neurons have already found good parameter values and still retain their own ability to adjust; injecting noise into these neurons may degrade performance. Our ReSiN method precisely identifies neurons on a per-neuron basis in a bidirectional manner, i.e., it targets neurons that neither contribute to the output nor can be effectively adjusted through gradients. More accurate localization of dormant neurons enables a more precise pruning of the solution space, which in turn leads to a tighter performance bound, as illustrated in Fig.~\ref{intuitive_illustration} (c).

\subsection{Dormant Neuron: Theoretical Analysis and Characterization}

In the PPO algorithm, both the actor and critic networks are implemented as neural networks that map state input $x_{t}$ to action output $a_{t}$ and value predictions $V(x_{t})$, respectively. To analyze neuron plasticity within these networks, we consider a feed-forward neural network layer $l$ characterized by its input space $\mathcal{X}_l \subseteq \mathbb{R}^{k_l}$, parameterized by $\mathbf{W}_l \in \mathbb{R}^{k_{l} \times k_{l+1}}$ and $\mathbf{b}_l \in \mathbb{R}^{k_{l+1}}$, and with an activation function $\sigma_l: \mathbb{R} \to \mathbb{R}$ applied elementwise. The forward pass of layer $l$ is $f_l(\mathbf{x}) = \sigma_l(\mathbf{W}^{T}_{l} \mathbf{x} + \mathbf{b}_{l})$, and the output of the $i$-th neuron in layer $l$ is $h_{l, i}(\mathbf{x})=\sigma_l\left(\mathbf{w}_{l, i}^T \mathbf{x}+b_{l, i}\right)$, where $\mathbf{w}_{l,i}^T$ is the $i$-th column of $\mathbf{W}_l$. Based on previous research \cite{sokar2023dormant,xu2024drm}, the dormant neuron can be defined in Definition \ref{dormant_neuron}.

\begin{definition}[{\textbf{Dormant Neuron}}]
\label{dormant_neuron} For a layer $l$ with $H_{l}$ neurons, the dormancy index $s_{l,i}$ of neuron ($l,i$) is defined as:

\begin{equation}
    s_{l,i} = \frac{\mathbb{E}_{\mathbf{x}\in D}|h_{l,i}(\mathbf{x})|}{\frac{1}{H_{l}}\sum_{j=1}^{H_{l}}\mathbb{E}_{\mathbf{x}\in D}|h_{l, j}(\mathbf{x})|},
\end{equation}
where $\mathbb{E}_{\mathbf{x} \in D}$ represents expectation over inputs $\mathbf{x}$ drawn from distribution $D$.
\end{definition}


Fig.~\ref{intuitive_illustration} (d) illustrates the data flow within the PPO neural network. Traditionally, a dormant neuron is defined as one whose forward activation is close to zero relative to other neurons. However, we argue that a near-zero forward activation does not necessarily indicate true inactivity. As long as the gradient flowing through that neuron is non-zero, the neuron can still contribute to learning and should therefore be considered active. To address this limitation, we introduce the concept of the Silent Neuron. A Silent Neuron is characterized by both zero forward activation and zero backward gradient, indicating that it is genuinely inactive during training. Such neurons can be safely reinitialized without affecting the current learning dynamics, making them useful for handling subsequent tasks and mitigating capacity saturation in non-stationary adaptive video streaming settings.

We next provide a theoretical discussion of Dormant Neurons and Silent Neurons to more rigorously characterize their behaviors and clarify the relationship between the two. To characterize the theoretical properties of dormant neurons, we impose the following mild regularity conditions. These assumptions follow standard regularity conditions commonly adopted in neural network theory, including smoothness of activations \cite{hornik1991approximation,cybenko1989approximation}, boundedness of neuron outputs used in generalization and stability analyses \cite{hardt2016train}, and non-degeneracy conditions preventing representational collapse \cite{santurkar2018does}. Before presenting the formal assumptions, we note that these regularity conditions are naturally satisfied by the ABR system model defined in Section \ref{StochasticNetworkResourceAdaptationProblem}. Specifically, the continuity assumption holds for all activation functions used in our PPO implementation; the boundedness assumption is guaranteed by the physical constraints of the ABR system, where buffer occupancy $B_{f}(t) \in [0, B_{f}^{max}]$ and network throughput measurements are normalized by construction; and the non-degeneracy assumption is empirically verified in Fig. \ref{dormantandzerogradient}, which confirms that the average layer-wise neuron activation remains strictly positive throughout all training conditions. These assumptions are therefore not restrictive idealizations but faithful reflections of the ABR learning system.

\begin{assumption}[\textbf{Continuity and Differentiability}]
\label{continuity} The neuron output $h_{l,j}(\mathbf{x})$ is continuously differentiable with respect to $\mathbf{x}$: $h_{l,j} \in C^{1}(\mathbb{R}^{k_{l}})$.
\end{assumption}

\begin{assumption}[\textbf{Boundedness}]
\label{boundedness} There exists $M > 0$ such that $\mathbb{E}_{\mathbf{\mathbf{x}} \in D}|h_{l,j}(\mathbf{\mathbf{x}})| \leq M$ for all $j \in h$.
\end{assumption}

\begin{assumption}[\textbf{Non-degeneracy}]
\label{non_degeneracy}
There exists $m > 0$ such that $\frac{1}{H_{l}}\sum_{j=1}^{H_{l}}\mathbb{E}_{\mathbf{x}\in D}|h_{l,j}(\mathbf{x})| \geq m$, ensuring a strictly positive denominator in the dormancy index.
\end{assumption}

\begin{theorem}[\textbf{Bidirectional Dormancy Characterization}]
\label{BidirectionalDormancyCharacterization}
Let $D \subseteq \mathcal{X}_l \subseteq \mathbb{R}^{k_l}$ be the domain from which inputs $\mathbf{x}$ are drawn. Consider layer $l$ with $H_l$ neurons, and let the activation of the $i$-th neuron be $h_{l,i}(\mathbf{x}) = \sigma_l(\mathbf{w}_{l,i}^{T} \mathbf{x} + b_{l,i})$, where $\sigma_l$ is continuously differentiable and applied elementwise. The neuron’s \emph{dormancy index} is defined by $s_{l,i} = \frac{\mathbb{E}_{\mathbf{x}\in D}\bigl|h_{l,i}(\mathbf{x})\bigr|}{\frac{1}{H_l}\sum_{j=1}^{H_l} \mathbb{E}_{\mathbf{x}\in D}\bigl|h_{l,j}(\mathbf{x})\bigr|}$.

Under Assumptions~\ref{continuity}–\ref{non_degeneracy}, the following statements are equivalent:
\begin{itemize}
    \item[(A)] \textbf{Dormancy:} $s_{l,i}=0$, equivalently $\mathbb{E}_{\mathbf{x}\in D}|h_{l,i}(\mathbf{x})|=0$.
    \item[(B)] \textbf{Zero gradient on $D$ and one zero activation:}
    \[
    \nabla h_{l,i}(\mathbf{x}) = 0 \quad \forall\, \mathbf{x}\in D,
    \qquad
    \exists\, \mathbf{x}_0 \in D : h_{l,i}(\mathbf{x}_0)=0.
    \]
\end{itemize}
\end{theorem}
\begin{proof} See Appendix \ref{BidirectionalDormancyCharacterizationTheoremProof}.
\end{proof}

Theorem \ref{BidirectionalDormancyCharacterization} shows that our dormancy index $s_{l,i}$ exactly captures when a neuron is truly inactive. If $s_{l,i}=0$, the neuron is silent in both the forward and backward passes—it outputs zero everywhere on the data domain and receives no gradient. Conversely, if a neuron has zero gradient everywhere and outputs zero at least once, then it must be identically zero on the entire domain, which again implies $s_{l,i}=0$. Thus, the theorem establishes that statistical dormancy and functional silence are equivalent, forming the theoretical basis for identifying neurons that can be safely reset or reused as Silent Neurons.

While the theoretical characterization of dormant neurons offers meaningful insights into the mechanisms underlying plasticity loss, it also comes with several important limitations. Firstly, the characterization of dormant neurons heavily relies on the choice of input distribution $D$. This means that different distributions may lead to varying conclusions about the presence of dormant neurons. Secondly, the zero-activation condition (B) is only required to hold over the distribution $D$, rather than the entire input space $\mathcal{X}_l$. This could potentially limit the robustness of detecting dormancy. Additionally, in practical implementations, dormant neurons are typically identified when their dormancy index $s_{l,i}$ falls below a small threshold $\epsilon$, rather than requiring exact zero values. This introduces additional complexity to the theoretical guarantees. To more precisely identify and intervene on dormant neurons under these practical limitations, we introduce the Reset Silent Neurons method.

\subsection{Silent Neurons and Reset Mechanism}

While the Dormant Neuron metric captures neural inactivity from the perspective of forward activations, its implications for backward propagation hold only under restrictive assumptions. To overcome these limitations and enable a more complete understanding of neural network dynamics, we propose an approach that directly examines gradient behavior in response to network inputs. $g_{\mathbf{w}} = \nabla_{\mathbf{w}}\sum_{n}f_{\mathbf{w}}(x_n)$, where $f_{\mathbf{w}}$ denotes the neural network function and $x_{n}$ represents input states in a batch. 

This formulation provides three key advantages over the gradient structure used in PPO. \textbf{First}, as a loss-independent gradient, it directly evaluates the network's sensitivity to its parameters, avoiding the masking effects that occur in loss-based gradients, where neurons may appear inactive due to zero loss gradients despite exhibiting substantial forward activations. \textbf{Second}, by aggregating gradients over all outputs $\sum_{n}f_{\mathbf{w}}(x_n)$, it provides a comprehensive assessment of neuron activity that reflects their contributions across the full decision space, thereby avoiding biases introduced by task-specific optimization objectives. \textbf{Third}, this approach promotes broader state-space coverage, enabling the detection of neurons that are essential for particular state representations but may not influence the current loss. 

From the network-level gradient $g_{\mathbf{w}}$, we summarize the backward activity of an individual neuron $i$ in layer $l$ by the mean absolute gradient over its incoming weights,
\begin{equation}
    \label{backeqdefine}
    g_{l, i}(\mathbf{x})=\frac{1}{k_l} \sum_{j^{\prime}=1}^{k_l}\left|\frac{\partial}{\partial W_{l, i, j^{\prime}}} \sum_n f_{\mathbf{w}}\left(x_n\right)\right|
\end{equation}
where $W_{l, i, j^{\prime}}$ is the weight connecting the $j^{\prime}$-th input to neuron $i$ in layer $l$, and $k_{l}$ is its number of incoming connections. Based on this analysis, we formally define Silent Neurons as follows:
\begin{definition}[{\textbf{Silent Neuron}}]
\label{SilentNeuron}
Let $H_{l}$ be the number of neurons in layer $l$. The activity index $\xi_{l, i}$ for neuron $(l, i)$ is defined by both its forward output and backward gradient:
\begin{equation}
\xi_{l,i} = \frac{\mathbb{E}_{\mathbf{x}\in D}|h_{l,i}(\mathbf{x})|}{\frac{1}{H_{l}}\sum_{j=1}^{H_{l}}(\mathbb{E}_{\mathbf{x}\in D}|h_{l,j}(\mathbf{x})|)} + \mathbb{E}_{\mathbf{x}\in D}|g_{l,i}(\mathbf{x})|,
\end{equation}
where $h_{l,i}(\mathbf{x})$ denotes the output of neuron $(l,i)$ for input $\mathbf{x}$, and $g_{l,i}(\mathbf{x})$ represents its backward activity defined in Eq. (\ref{backeqdefine}). The forward indicator $\xi^{d}_{l,i}:= \frac{\mathbb{E}_{\mathbf{x}\in D}|h_{l,i}(\mathbf{x})|}{\frac{1}{H_l}\sum_{j=1}^{H_l}\mathbb{E}_{\mathbf{x}\in D}|h_{l,j}(\mathbf{x})|}$ is normalized by the layer-average activation, whereas the backward indicator $\xi^{g}_{l,i}:=\mathbb{E}_{\mathbf{x}\in D}\,|g_{l,i}(\mathbf{x})|$ retains its absolute magnitude. 
Given global thresholds $\delta_{g},\delta_{d}>0$, neuron $(l,i)$ is silent if it satisfies \emph{both} criteria simultaneously; the set of silent neurons at iteration $t$ is $\mathcal{S}_t = \bigl\{(l,i):\ \xi^{g}_{l,i}\le\delta_{g},\ \ \xi^{d}_{l,i}\le\delta_{d}\bigr\}$. The two indicators are aggregated by intersection: a neuron is reset only when its backward activity and its forward activity are both below their respective thresholds.
\end{definition}

In contrast to the definition of Dormant Neurons in \textit{Definition~\ref{dormant_neuron}}, Silent Neurons additionally require the backward gradient $\mathbb{E}_{\mathbf{x}\in D}|g_{l,i}(\mathbf{x})|$ to vanish, thereby incorporating an explicit criterion on gradient inactivity. We keep the backward term unnormalized so that it reflects a neuron's absolute functional sensitivity. Since the two indicators are combined by intersection, the backward criterion only refines the layer-normalized forward indicator, so a separate gradient normalization is unnecessary. Building on this definition, \textit{Theorem~\ref{SilentNeuronCharacterization}} characterizes the equivalence between a small activity index and the simultaneous vanishing of forward and backward activity.

\begin{theorem}[\textbf{Silent Neuron Characterization}]
\label{SilentNeuronCharacterization}
Let $D \subseteq \mathcal{X}_l \subseteq \mathbb{R}^{k_l}$ be the domain from which inputs $\mathbf{x}$ are drawn. Consider layer $l$ with $H_{l}$ neurons. For the $i$-th neuron in layer $l$, let $h_{l, i}(\mathbf{x})=\sigma_l\left(\mathbf{w}_{l, i}^T \mathbf{x}+b_{l, i}\right)$, where $\sigma_{l}$ is continuously differentiable and applied elementwise. The neuron's activity index is defined as:
\begin{equation}
\xi_{l,i} = \frac{\mathbb{E}_{\mathbf{x}\in D}|h_{l,i}(\mathbf{x})|}{\frac{1}{H_{l}}\sum_{j=1}^{H_{l}}(\mathbb{E}_{\mathbf{x}\in D}|h_{l, j}(\mathbf{x})|)} + \mathbb{E}_{\mathbf{x}\in D}|g_{l,i}(\mathbf{x})|, \nonumber
\end{equation}
Under Assumptions~\ref{continuity}–\ref{non_degeneracy}, suppose there exist constants $M_h, M_g > 0$ such that $\mathbb{E}_{\mathbf{x}\in D}|h_{l,i}(\mathbf{x})| \le M_h$, $\mathbb{E}_{\mathbf{x}\in D}|g_{l,i}(\mathbf{x})| \le M_g$. Then, for an arbitrarily small constant $\epsilon > 0$, the following statements are equivalent:
\begin{itemize}
\item (A) \textbf{Silence}: For any $\epsilon > 0$, the activity index satisfies $\xi_{l,i} < \epsilon$.
\item (B) \textbf{Vanishing forward and backward activity (on $D$)}: For any $\delta_{d} > 0$ and $\delta_{g} > 0$, we have $\mathbb{E}_{\mathbf{x} \in D}|h_{l, i}(\mathbf{x})| < \delta_{d}$ and $\mathbb{E}_{\mathbf{x} \in D}|g_{l, i}(\mathbf{x})| < \delta_{g}$.
\end{itemize}
\end{theorem}
\begin{proof}
See Appendix \ref{SilentNeuronCharacterizationTheoremProof}.
\end{proof}

Theorem~\ref{SilentNeuronCharacterization} extends our understanding of neural inactivity by addressing key limitations in the traditional notion of dormant neurons. The Silent Neuron framework provides a more flexible and realistic alternative. By introducing the activity index $\xi_{l, i}$, which jointly evaluates forward activations and backward gradients over the distribution $D$, we replace the strict zero requirement with a threshold-based condition $\xi_{l, i} < \epsilon$. This relaxation offers two key advantages: it better captures near-silent neurons that arise in practice, where exact zeros are uncommon, and it yields greater robustness across varying input distributions by emphasizing consistently low activation and gradient patterns. Notably, simply relaxing the traditional dormancy condition from $s_{l, i} = 0$ to $s_{l, i} < \epsilon$ does not guarantee the gradient properties required by our theory. This motivates the Silent Neuron formulation, which explicitly incorporates both activation and gradient information.

\subsection{The Silent Neuron–Enhanced PPO Algorithm}
\label{sec:resin_ppo}

While practical PPO involves coupled actor-critic dynamics, we follow \cite{he2025plasticity} and analyze the tracking performance of our update rule under a general non-stationary objective $L_t(w)$  to isolate the impact of silent-neuron-based reinitialization. Silent neurons, characterized by simultaneously low forward activations and backward gradients, effectively represent units with severely reduced plasticity. Their negligible contribution to both computation and gradient-based learning makes them natural targets for controlled parameter perturbation and reset. We now integrate this mechanism into PPO and obtain a Silent Neuron–Enhanced PPO algorithm. Let $\mathbf{w}_t \in \mathbb{R}^d$ denote the flattened policy and value parameters at iteration $t$, and let $L_t(\mathbf{w}_t)$ be the standard PPO surrogate objective defined in \cite{schulman2017proximal}. The vanilla PPO update can be written as $\mathbf{w}_{t+1} = \mathbf{w}_t- \eta \,\nabla L_t(\mathbf{w}_t)$, where $\eta>0$ is the learning rate.

\paragraph{Projection onto the silent subspace.}
At iteration $t$, we identify the set of silent neurons $\mathcal{S}_t$ according
to Definition~\ref{SilentNeuron}, and associate it with an orthogonal projection
matrix $\Pi_t \in \mathbb{R}^{d\times d}$ satisfying
\[
\Pi_t^2 = \Pi_t, 
\qquad
\Pi_t^{T} = \Pi_t,
\qquad
\mathrm{rank}(\Pi_t) = d_{s,t},
\]
where $d_{s,t}$ is the dimension of the silent neuron subspace at iteration $t$, and in practice $d_{s,t} \ll d$.

\paragraph{Update rule: modeling resets as subspace perturbations.}

Reinitializing the parameters associated with silent neurons can be analytically treated as adding a stochastic perturbation restricted to the corresponding coordinate subspace \cite{frankle2018the}. This perturbation mechanism is used in PA-MoE \cite{he2025plasticity}, where controlled noise is injected into selected coordinates to restore plasticity. Formally, instead of injecting isotropic noise into the full parameter space, we restrict stochastic perturbations to the silent neuron subspace:
\begin{equation}
    \mathbf{w}_{t+1}
    =
    \mathbf{w}_t
    - \eta \,\nabla L_t(\mathbf{w}_t)
    + \eta \gamma \Pi_t \boldsymbol{\epsilon}_t,
    \qquad
    \boldsymbol{\epsilon}_t \sim \mathcal{N}(0, I_d),
    \label{eq:resin_update}
\end{equation}
where $\gamma>0$ controls the perturbation strength and $\Pi_t$ is the projection onto the coordinates associated with silent neurons. The noise is still sampled from a standard Gaussian in $\mathbb{R}^d$, but only the projected component $\Pi_t \boldsymbol{\epsilon}_t$ affects the parameters. Consequently, stochastic exploration is concentrated exclusively on neurons identified as silent, while leaving active neurons unaffected. Let the resulting perturbation term be $E_t := \eta\gamma \Pi_t \boldsymbol{\epsilon}_t$. This formulation allows reset operations to be analyzed within the same mathematical framework as subspace-restricted noise injection, enabling convergence-style bounds analogous to those established for PA-MoE \cite{he2025plasticity}.

\begin{theorem}[\textbf{Tracking Analysis of Subspace-Restricted Updates for PPO}]
\label{thm:resin-error-bound}
Under Assumptions (A4)--(A6) with step-size $0 < \eta \leq 1/L$, and with the update rule, $\boldsymbol{w}_{t+1} = \boldsymbol{w}_t - \eta\, \nabla L_t\bigl(\boldsymbol{w}_t\bigr) + \eta\, \gamma\, \Pi_{t}\epsilon_t, \quad \epsilon_t \sim \mathcal{N}\left(0, I_d\right)$, the average squared error satisfies, $\frac{1}{T} \sum_{t=1}^{T} \mathbb{E}\|e_t\|^2 \le \frac{2}{\mu \eta T} \left(\mathbb{E}||e_{0}||^{2} + \frac{2 P_{T}^{2}}{\mu \eta} \right) + \frac{2 \eta \gamma^{2}}{\mu} \cdot \frac{1}{T}\sum_{t=0}^{T-1}d_{s,t}$, where the error is defined as $e_t = \boldsymbol{w}_t - \boldsymbol{w}_t^*$, and the path length of the optimal parameters is $P_T = \sum_{t=1}^{T-1}\|\boldsymbol{w}_{t+1}^* - \boldsymbol{w}_t^*\|$.
\end{theorem}
\begin{proof}
See Appendix \ref{TheSilentNeuronEnhancedPPOPerformanceBound}.
\end{proof}

Compared to the PA-MoE's bound $\frac{2 \eta \gamma^{2}}{\mu} \cdot d$, and for the same algorithmic parameters, the new expression replaces the global worst-case constant $d$ by the empirical average $\frac{1}{T}\sum_{t=0}^{T-1} d_{s,t}$. Hence, the resulting upper bound is always no larger and is strictly tighter whenever $d_{s,t} < d$ for some $t$.

\paragraph{Computational Complexity.}

The additional computational cost introduced by ReSiN-PPO remains small relative to PPO training. Forward-activity scores $\xi^{d}_{l,i}$ are collected as part of the normal forward computation and incur no extra cost. After the initial warm-up period of $F$ iterations, computing backward-activity scores $\xi^{g}_{l,i}$ requires one additional backward pass per iteration to obtain neuron-wise gradient magnitudes, which has the same asymptotic complexity as a standard PPO backward pass. Identifying silent neurons and constructing the projection (or mask) over the associated parameters requires only element-wise thresholding and masking operations, both of which scale linearly with the number of trainable parameters, i.e., $O(d_{\mathbf{w}})$. Consequently, the per-iteration overhead in the reset phase is $O(d_{\mathbf{w}})$, matching the complexity of a standard PPO update up to a small constant factor. In practice, this overhead is modest because the extra backward pass shares the same computational graph as the PPO backward pass, and the thresholding and masking operations are lightweight element-wise operations. Since ReSiN runs online, decision-making and parameter updates can be decoupled using separate weights to avoid added decision latency \cite{racanelli2025real}.

\paragraph{Resin-PPO Algorithm}
Algorithm~\ref{alg:resin_ppo} summarizes the overall Silent Neuron–Enhanced PPO procedure. After $F$ warm-up iterations, the algorithm evaluates neuron activity, constructs the projection $\Pi_t$, and applies the selective perturbation update in Eq.~\eqref{eq:resin_update}. In practice, $\Pi_t$ is implemented as a binary mask over parameters associated with silent neurons.

\begin{algorithm}[!htbp]
\caption{Silent Neuron–Enhanced PPO (ReSiN-PPO)}
\label{alg:resin_ppo}
\begin{algorithmic}[1]
\REQUIRE Policy and value network parameters $\mathbf{w}$, 
learning rate $\eta$, activity thresholds $\delta_{g}, \delta_{d}$, 
maximum training iterations $T$, 
init reset delay $F$.
\FOR{iteration $t =1$ to $T$}
    \STATE Interact with the environment and collect trajectories.
    \STATE Compute PPO loss $L_t(\mathbf{w})$ and gradient $g_t = \nabla_{\mathbf{w}}. L_t(\mathbf{w})$
    \IF{$t > F$}
        \STATE Compute neuron activity scores $\xi^{g}_{l,i}$ and $\xi^{d}_{l,i}$ for all hidden units,
        \STATE Identify neurons 
        $\mathcal{S}_t = \{(l,i): \xi^{g}_{l,i} \le \delta_{g},\, \xi^{d}_{l,i} \le \delta_{d}\}$,
        \STATE Construct projection $\Pi_t$ (or mask) onto parameters associated with $\mathcal{S}_t$.
        \STATE \textbf{Reset silent neurons:} reinitialize the parameters of neurons in
        $\mathcal{S}_t$ from the initialization distribution, i.e. $\mathbf{w} \leftarrow (\mathbf{I}-\Pi_t)\,\mathbf{w} + \Pi_t\,\mathbf{w}_{\mathrm{init}}$.
        \STATE \textbf{PPO update:} $\mathbf{w} \leftarrow \mathbf{w} - \eta g_t$.
    \ELSE
        \STATE \textbf{Standard PPO update:}
        $\mathbf{w} \leftarrow \mathbf{w} - \eta g_t$
    \ENDIF
\ENDFOR
\end{algorithmic}
\end{algorithm}

\section{Experiments}
\label{experiment}

Our experiments are divided into two parts. The first part examines the impact of abrupt switching between high and low network bandwidth on plasticity. This is done by analyzing the performance degradation of the ABR system under non-stationary conditions and observing changes in the proportion of dormant neurons in the agent's internal network. This is to verify our theoretical hypothesis. The second part evaluates the effectiveness of our proposed method in improving performance under subtle, gradual, unpredictable, and diverse non-stationarity found in real-world bandwidth traces. We also test whether our Silent Neuron hypothesis holds true under these conditions. Additionally, we assess the generalizability of our method to more advanced network architectures and compare its performance against PA-MoE to confirm the consistency with our theoretically tighter bound.

\subsection{Experiment Setting}

The experiments were performed on a system equipped with an Intel(R) Core(TM) i5-10400 CPU @ 2.90GHz, without GPU acceleration. We train the agent with a learning rate of $1 \times 10^{-4}$. Each training iteration collects 2,000 rollout steps, which form a batch of size 2,000. This batch is further divided into minibatches of size 62 and optimized for 5 update epochs per iteration. For the advantage estimation, we use generalized advantage estimation (GAE) with discount factor $0.99$ and GAE parameter $\lambda=0.95$. Following the implementation in \cite{dohare2024loss}, our initial reset delay is $F=1000$, which provides sufficient warm-up time for the PPO agent to establish stable neural activation patterns before Silent Neuron identification begins. Unless otherwise stated, the thresholds for silent neuron detection are fixed to $\delta_{g}=0.05$ and $\delta_{d}=0.5$ in all experiments. In practice, the perturbation of silent neurons is realized by re-initializing their parameters using LeCun uniform initialization \cite{lecun2002efficient}. In the performance comparison experiments, we use the same network architecture as Pensieve \cite{mao2017neural} and Merina \cite{kan2025merina+}. We use the EnvivioDash3 sequence and its per-chunk size traces as released by Pensieve~\cite{mao2017neural}; the ABR system parameters are listed in Table~\ref{tab:parameters}.
\begin{table}[!htbp]
\centering
\caption{ABR System parameters.}
\label{tab:parameters}
\begin{tabular}{ll}
\toprule
Parameter & Value \\
\midrule
Chunk duration $t_i$        & $4$ s \\
Chunks per session          & $48$ \\
Bitrate ladder $b_i$        & $\{0.3, 0.75, 1.2, 1.85, 2.85, 4.3\}$ Mbps \\
Buffer capacity $B_f^{max}$ & $60$ s \\
Round-trip time             & $80$ ms \\
\midrule
$\alpha$ (quality offset)   & $0.7$ \\
$\beta$ (bitrate penalty)   & $0.01$ \\
$\mu_1$ (smoothness weight) & $0.5$ \\
$\mu_2$ (rebuffering weight)& $4.3$ \\
\bottomrule
\end{tabular}
\end{table}

\textbf{Network Trace Datasets:} Follow \cite{he2025plasticity}, our evaluation uses throughput traces collected from three different real-world sources. First, the FCC dataset \cite{kan2022improving} is constructed by concatenating randomly selected logs belonging to the “Web browsing" category from the August 2016 public release. Second, we utilize the Puffer dataset \cite{yan2020learning}, which provides measurements from on-demand video streaming over a wide range of access technologies, such as wired networks, cellular systems (3G/4G/5G), and Wi-Fi. Third, we employ a set of HSDPA-based 3G traces \cite{riiser2013commute}, originally collected from mobile devices streaming videos in transit on subways, buses, trains, ferries, and trams. We split these traces into 127 samples for training and 142 for testing, with no shared elements between the two splits. Among the $127$ training traces, those with an average throughput below $2$~Mbps are categorized as low-bandwidth (LBW, $53$ traces) and those above $3$~Mbps as high-bandwidth (HBW, $51$ traces); the remaining $23$ traces, with average bandwidth between 2 and 3~Mbps, are excluded from both groups \cite{akhtar2018oboe}. These two groups are used for the experiments on non-stationary network bandwidth.

\subsection{Non-stationarity System Analysis}

What is the reason for the decrease in performance shown in Fig. \ref{ppo_bitrate_comparison} (a) when there is a sudden switch between high and low bandwidth? In this subsection, we will explore this question by analyzing the connection between the ABR system's internal state variables and the various components of QoE performance.

\begin{figure*}[!htbp]
\centering
\includegraphics[width=0.7\textwidth]{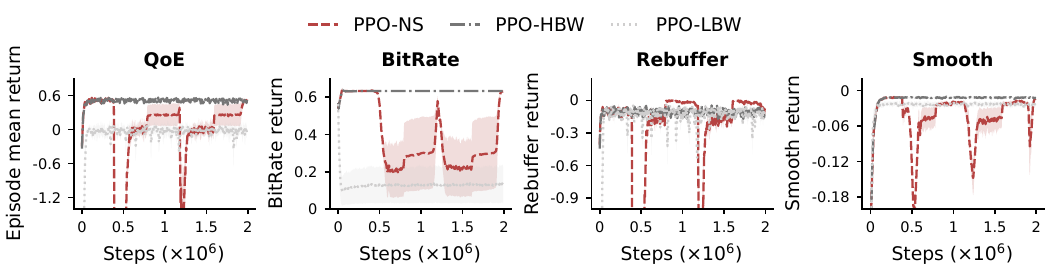}%
\caption{Performance comparison of PPO under different network conditions in terms of QoE metrics. The figure shows the learning curves of the total QoE reward and its components.}
\label{ppo_seperate_reward}
\end{figure*}

Fig.~\ref{ppo_seperate_reward} illustrates the significant impact of non-stationarity on the learning system's plasticity. The agent is unable to regain its pre-transition performance across all QoE metrics, indicating a persistent loss of adaptability. The agent shows poor adaptation during network transitions in both directions. When bandwidth decreases, it continues to select high bitrates despite reduced capacity. Conversely, when bandwidth increases, it fails to take advantage of the additional resources, resulting in low bitrates and a sudden surge in rebuffer time and reward. This bidirectional adaptation failure highlights the system's state during periods of non-stationarity: a fundamental loss of plasticity, leaving the agent unable to effectively respond to dynamic network conditions. This rigidity leads to suboptimal performance and a degraded user experience.

\begin{figure}[!htbp]
\centering
\includegraphics[width=\columnwidth]{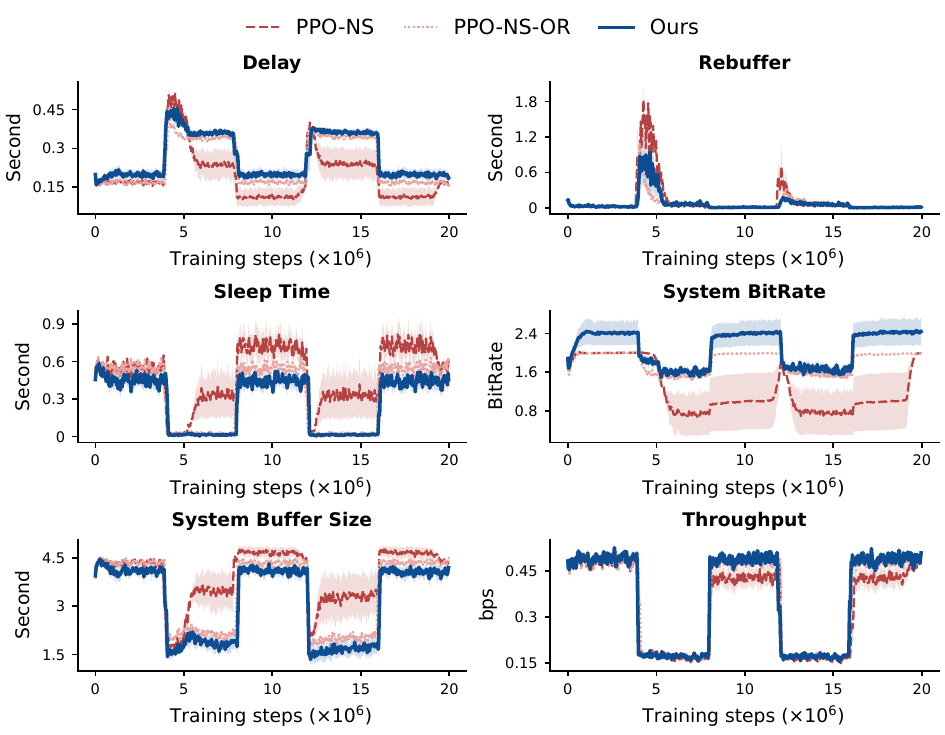}%
\caption{System metrics under different plasticity maintenance strategies in non-stationary environments are compared in the figure. The strategies include standard PPO (PPO-NS), PPO with output-based reset using dormant neuron detection (PPO-NS-OR), and our proposed method of PPO with intersection-based reset using the proposed Silent Neuron criterion (Ours).}
\label{system_insight_info}
\end{figure}

Fig.~\ref{system_insight_info} reveals how non-stationarity affects system behavior and resource utilization. Network transitions trigger abrupt changes, notably in rebuffering duration. The agent struggles to adapt, leading to suboptimal resource utilization: excessive buffering, extended sleep intervals, and underutilized throughput. This translates to degraded user experience, with the system exhibiting either over-conservative (excessive buffering and sleep time) or over-aggressive (increased rebuffering) behavior during transitions, indicating an inability to maintain optimal operation under changing conditions.

\subsection{Neural Plasticity Loss and Adaptation Capability}

In this section, we will analyze from the perspective of neuron outputs and gradients in order to understand why neural networks struggle to quickly adapt and learn in non-stationary environments. Fig.~\ref{dormantandzerogradient} provides important insights into the loss of neural plasticity. It is evident from the figure that once neurons become dormant, they tend to remain inactive in subsequent iterations, as indicated by the overlap between the dormant neuron curves and their persistence metrics. This phenomenon can be described as a "plasticity trap", where neurons lose their ability to adapt to new conditions in network resource adaptation tasks.

\begin{figure}[!htbp]
\centering
\includegraphics[width=\columnwidth]{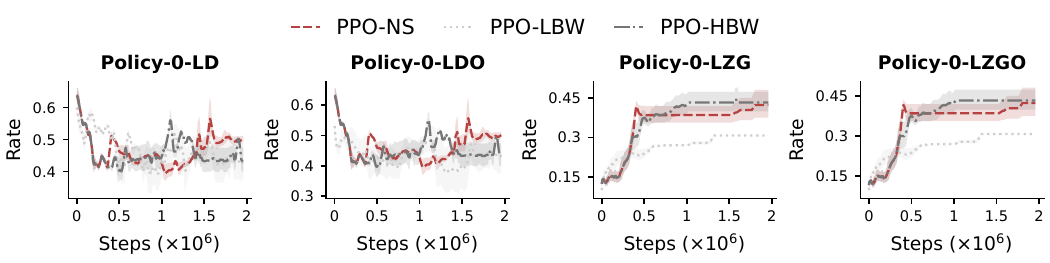}
\includegraphics[width=\columnwidth]{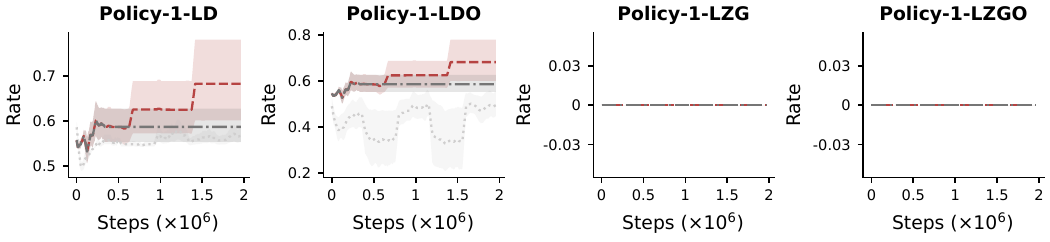}
\includegraphics[width=\columnwidth]{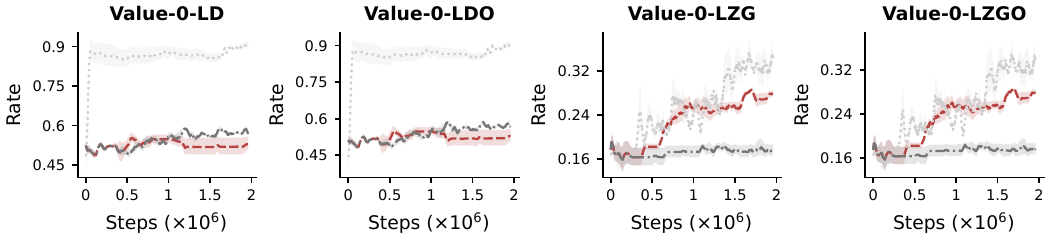}
\includegraphics[width=\columnwidth]{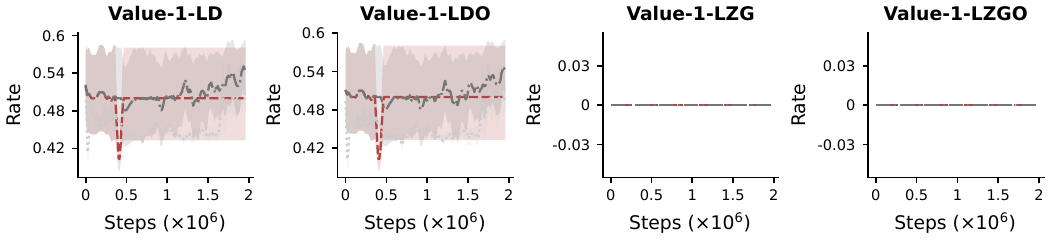}
\caption{The activity patterns of neurons across network layers and components, such as the ratio of dormant neurons (Layer-Dormant, LD), the persistence of dormant neurons (DormantOverLep, LDO), the ratio of zero gradients (ZeroGradient, LZG), and the persistence of zero gradients (ZeroGradientOverLep, LZGO) are examined for both policy (Policy-0, Policy-1) and value (Value-0, Value-1) networks.}
\label{dormantandzerogradient}
\end{figure}

Fig.~\ref{dormantandzerogradient} also illustrates an increasing proportion of dormant and zero-gradient neurons during training in non-stationary environments, indicating a progressive loss of plasticity. Interestingly, similar patterns emerge in both high and low bitrate scenarios, suggesting that plasticity loss is inherent in RL tasks where expected future reward distributions evolve and network conditions fluctuate. The analysis reveals layer-wise differences in plasticity loss: the shallow layer (Layer 0) shows more severe loss compared to the deeper layer (Layer 1) in both policy and value networks. This suggests that the network's basic feature processing is more vulnerable to plasticity loss than higher-level processing in adaptive video streaming tasks. These findings identify neural plasticity loss as a fundamental, progressive challenge in non-stationary network adaptation, explaining why conventional approaches struggle under dynamic conditions.

\subsection{Performance Evaluation}

In this subsection, we quantitatively evaluate the performance gains of our silent neuron resetting method in the scenario of Fig.\ref{ppo_bitrate_comparison} (a).

\begin{figure}[!htbp]
\centering
\includegraphics[width=\columnwidth]{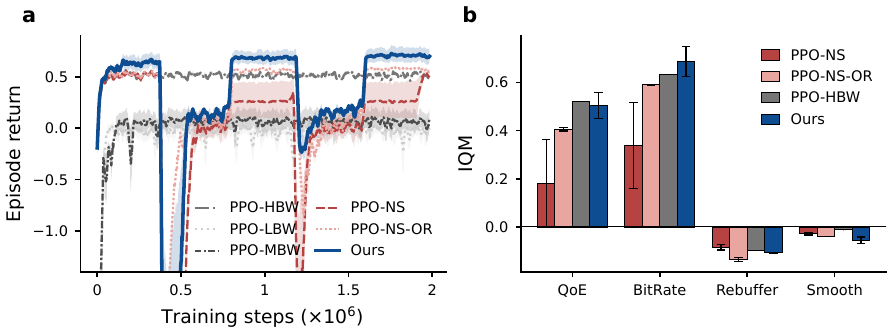}%
\caption{(a): Performance comparison of PPO variants across different network conditions and plasticity maintenance strategies. The figure shows the learning curves of QoE rewards for six scenarios: PPO-HBW, PPO-LBW, PPO in mixed-bandwidth environments (PPO-MBW), PPO-NS, PPO-NS-OR, and Ours. (b): Comparison of QoE Methods using the IQM  (Interquartile Mean) with a focus on the 25th-75th percentile returns.}
\label{final_performance_comparation}
\end{figure}

Fig.~\ref{final_performance_comparation} (a) shows that our Silent Neuron reset method outperforms conventional approaches, addressing plasticity loss caused by both non-stationary network conditions and inherent RL processes. Every configuration is trained with two random seeds ($1,2$) for $2\times10^{6}$ environment steps, and the per-episode QoE return is logged every $200$ steps. The solid line then reports the mean across the two seeds and the shaded band the $25$th--$75$th percentile across them. It surpasses the performance of the conventional approach in stable high-bandwidth conditions (PPO-HBW) and significantly outperforms the dormant-based reset method. These results highlight the importance of considering both forward outputs and backward gradients for robust adaptation in dynamic environments.

Further analysis of the reward components over time reveals the underlying mechanism of our method's superior performance, as shown in Fig.~\ref{final_performance_comparation} (b). Our Silent Neuron reset approach demonstrates a sophisticated balance in optimizing multiple QoE metrics simultaneously. Specifically, while maintaining low rebuffering events and ensuring smooth bitrate transitions, our method exhibits more aggressive yet stable bitrate selection behavior. This optimal trade-off is achieved through enhanced plasticity maintenance, allowing the agent to make more informed decisions without compromising system stability. Under the IQM computed over the entire training trajectory, our method attains an IQM QoE of $0.504$, improving over PPO-NS ($0.181$) by $178\%$ and over PPO-NS-OR ($0.406$) by $24\%$. PPO-HBW reaches a comparable IQM ($0.522$) only because it runs under stationary high-bandwidth conditions and stays nearly flat, whereas our method starts lower and improves steadily.

\subsection{Analysis of Neural Activity Patterns and Method Effectiveness}

\begin{figure}[!htbp]
\centering
\includegraphics[width=\columnwidth]{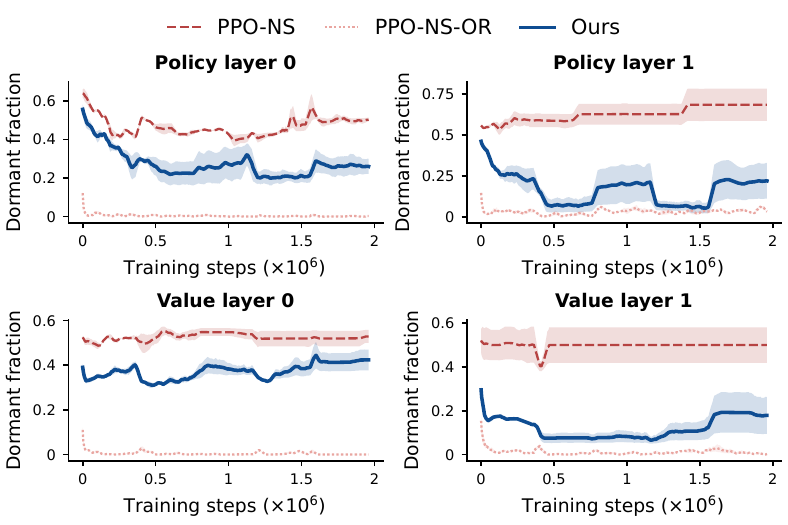}%
\caption{Dormancy patterns in policy and value networks, specifically in layers 0-1. It explores how plasticity maintenance strategies affect learning and neural dynamics.}
\label{correlation_dormant_reward}
\end{figure}

In this subsection, we analyze whether the observed performance gains align with our theoretical predictions. Fig.~\ref{correlation_dormant_reward} illustrates that dormant neurons do not necessarily indicate a loss of plasticity. Our Silent Neuron reset method produces more dormant neurons compared to output reset, but it still achieves better performance in non-stationary environments. This supports our theory that neurons with zero outputs but non-zero gradients can improve network expressiveness and adaptability, challenging the traditional belief that dormant neurons lead to reduced plasticity. These empirical results provide evidence for our theoretical framework on neural plasticity in deep RL systems.

\subsection{Ablation Experiment}
\label{ablation_experiment}

\begin{figure}[!htbp]
\centering
\includegraphics[width=\columnwidth]{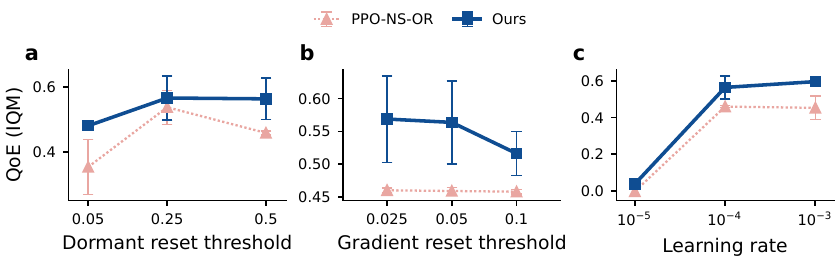}%
\caption{(a): Performance comparison under different dormant reset thresholds. (b): Performance comparison under different gradient Reset threshold. (c): Performance comparison under different learning rates.}
\label{ablation_dormant}
\end{figure}

To validate the effectiveness and generality of our proposed ReSiN framework, we conducted comprehensive ablation studies focusing on the Reset mechanism. In particular, we evaluated ReSiN's performance by varying the gradient reset thresholds and dormancy thresholds. To assess the overall performance of the algorithm, we conducted ablation studies using various learning rates. The result is shown in Fig. \ref{ablation_dormant}. Our proposed approach has consistently demonstrated superior QoE across a range of learning rates and reset thresholds. 

\subsection{Non-stationarity Analysis under Real-World Traces}

Fig. \ref{ppo_bitrate_comparison} (b) demonstrates that loss of plasticity also occurs in real-world scenarios. In light of this, we purposely avoid manually categorizing or altering trace bandwidth in this subsection to test the effectiveness of our proposed method under more realistic conditions.

\begin{figure*}[!htbp]
\centering
\includegraphics[width=\textwidth]{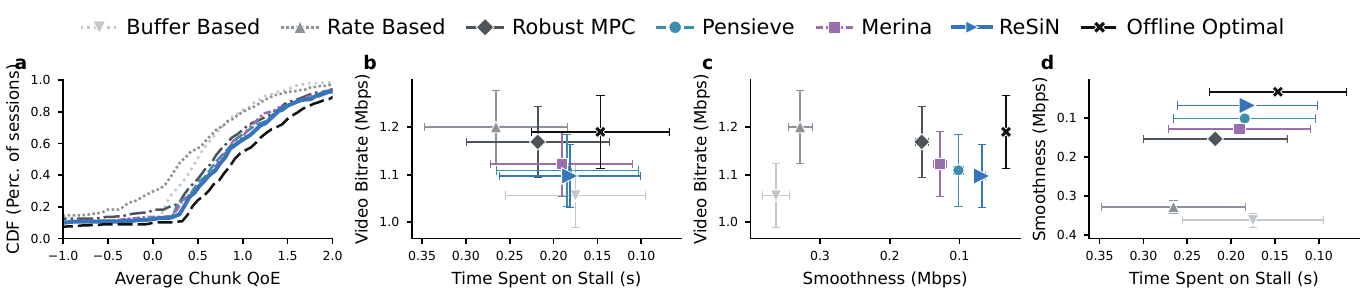}%
\caption{Comparing ReSiN with recent ABR algorithms over the Train dataset.}
\label{ppo_seperate_reward_train}
\end{figure*}

\begin{figure*}[!htbp]
\centering
\includegraphics[width=\textwidth]{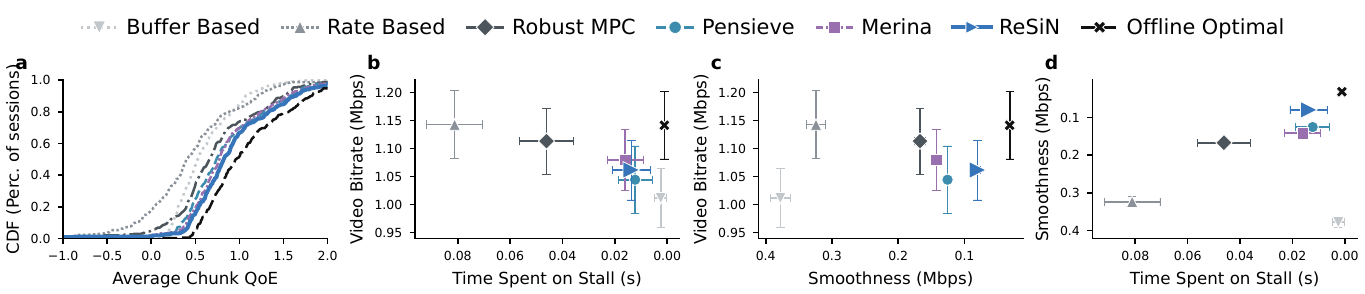}%
\caption{Comparing ReSiN with recent ABR algorithms over the Test dataset.}
\label{ppo_seperate_reward_test}
\end{figure*}

\begin{table*}[!htbp]
  \begingroup
  \centering
  \caption{Comparison of QoE components and mean inference time for different ABR algorithms on the test dataset (mean over episodes).}
  \label{tab:result}
  \begin{tabular}{lccccc}
    \toprule
    Algorithm &
    Mean QoE $\uparrow$ &
    Bitrate (Mbps) $\uparrow$ &
    Rebuffer Time (s) $\downarrow$ &
    Smoothness (Mbps) $\downarrow$ &
    Inference Time (ms) $\downarrow$ \\
    \midrule
    \multicolumn{5}{l}{\textbf{Rule-based methods}} \\
    Buffer Based & $0.623 \pm 0.379$ & $1.011 \pm 0.053$ & $0.002 \pm 0.002$ & $0.378 \pm 0.015$ & $0.0026 \pm 0.0000$\\
    Rate Based & $0.470 \pm 0.583$ & $1.143 \pm 0.060$ & $0.081 \pm 0.011$ & $0.324 \pm 0.014$ & $0.0089 \pm 0.0000$\\
    Robust MPC & $0.748 \pm 0.587$ & $1.113 \pm 0.058$ & $0.046 \pm 0.010$ & $0.167 \pm 0.008$ & $0.1075 \pm 0.0015$\\
    \addlinespace[0.4em]
    \multicolumn{5}{l}{\textbf{Learning-based methods}} \\
    Pensieve & $0.866 \pm 0.541$ & $1.044 \pm 0.060$ & $0.012 \pm 0.006$ & $0.126 \pm 0.006$ & $0.2984 \pm 0.0025$\\
    Merina & $0.869 \pm 0.506$ & $1.080 \pm 0.054$ & $0.016 \pm 0.007$ & $0.142 \pm 0.008$ & $0.4118 \pm 0.0010$ \\
    PA-MoE & $0.914 \pm 0.478$ & $1.063 \pm 0.051$ & $0.009 \pm 0.005$ & $0.109 \pm 0.004$ & $0.8056 \pm 0.0081$ \\
    ReSiN (Ours) & $\textbf{0.923} \pm 0.520$ & $1.062 \pm 0.053$ & $0.014 \pm 0.007$ & $0.080 \pm 0.005$ & $0.3620 \pm 0.0015$ \\
    ReSiN-MoE (Ours) & $\textbf{0.932} \pm 0.446$ & $1.041 \pm 0.051$ & $0.005 \pm 0.003$ & $0.087 \pm 0.004$ & $0.9242 \pm 0.0022$\\
    \multicolumn{5}{l}{\textbf{Upper Bound}} \\
    Offline Optimal & $1.105 \pm 0.524$ & $1.142 \pm 0.060$ & $0.001 \pm 0.001$ & $0.032 \pm 0.002$ & $-$ \\
    \bottomrule
  \end{tabular}
  \par\vspace{0.3em}
  {\footnotesize
  \textit{Note:} $\uparrow$ indicates larger is better; $\downarrow$ indicates smaller is better.
  Bitrate, Rebuffer Time, and Smoothness are reported as mean $\pm$ 95\% Confidence Interval (CI); Inference time is the average per-decision inference latency in milliseconds.}
  \endgroup
\end{table*}

Following Pensieve~\cite{mao2017neural} and Merina~\cite{kan2025merina+}, we adopt the same network architecture and QoE objective, with a linear quality term $q(b_i) = b_i$, $\mu_1 = 1$ and $\mu_2 = 4.3$. At the start of each video session, a trace is drawn uniformly at random from the training set with a random starting offset; the sequence wraps around within a session. All randomness is controlled by a single seed, making the trace order fully reproducible \cite{xia2022genet}. Experiments are conducted under two random seeds: the main text reports seed 0, and Appendix~\ref{AblationSeed1} reports seed 1 to verify robustness. Model weights are saved every 500 training steps; these checkpoints are then used for the subsequent evaluation \cite{he2026nsma}. The experimental outcomes on the training and testing traces are illustrated in Fig.~\ref{ppo_seperate_reward_train} and \ref{ppo_seperate_reward_test}. These results are obtained by selecting the best-performing checkpoint during training and evaluating it over 300 episodes.

Fig.~\ref{ppo_seperate_reward_train} visualizes the cumulative distribution functions (CDFs) of average QoE across all sessions and algorithms in the training split, together with a breakdown of the contributing QoE metrics—bitrate level, playback smoothness, and rebuffering duration. Fig.~\ref{ppo_seperate_reward_test} presents the analogous results on the test split. From these comparisons, we observe that ReSiN consistently delivers substantial QoE improvements on both datasets, and outperforms the meta–reinforcement learning approach Merina \cite{kan2025merina+}. Table \ref{tab:result} quantifies the final performance of different algorithms under real-world trace data. The Offline Optimal serves as an upper bound: for each episode, we compute the QoE-optimal bitrate sequence with full knowledge of the future bandwidth trace, solved exactly by Bellman recursion (over chunk index, buffer occupancy discretized at 0.1s, and previous bitrate) in the same simulator \cite{huang2019comyco}.

One potential concern is whether the performance improvements observed in real-world traces, which often exhibit subtle non-stationarity, align with our theoretical explanation. We evaluate the checkpoints obtained during training using five random seeds (0,1,2,3,4), with 300 episodes per seed, and record the average QoE and dormancy metrics. Averaged over the full trajectory, the mean QoE is $0.876 \pm 0.046$ (Pensieve), $0.877 \pm 0.045$ (ReSiN-FwdOnly), $0.854 \pm 0.044$ (ReSiN-BwdOnly) and $0.891 \pm 0.048$ (ReSiN Fwd+Bwd), with corresponding dormant fractions of $0.169 \pm 0.001$, $0.058 \pm 0.001$, $0.148 \pm 0.001$ and $0.081 \pm 0.001$ (mean $\pm$ 95\% CI). Relative to Pensieve, ReSiN (Fwd+Bwd) improves the mean QoE by $1.7\%$ ($0.876 \to 0.891$).

\begin{figure}[!htbp]  
  \centering
  \includegraphics[width=\columnwidth]{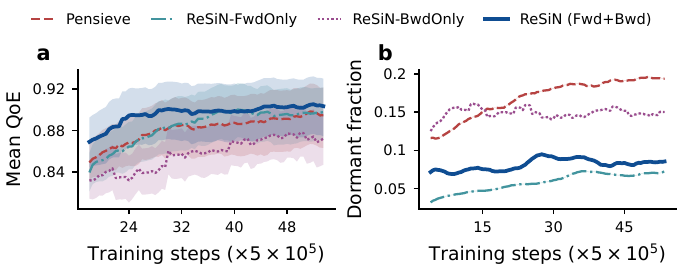}
  \caption{Performance gains and dormancy metrics under real-world trace training. The model is trained with seed 0 and evaluated on the test set.}
  \label{ppo_reward_ablation_comparison_seed0}
\end{figure}

Fig. \ref{ppo_reward_ablation_comparison_seed0} clearly demonstrates that Pensieve suffers from a lack of adaptability, and that resetting either dormant neurons or silent neurons can improve its performance. It is worth noting that resetting silent neurons (ReSiN, Fwd+Bwd), which includes a larger number of neurons with zero forward output (ReSiN, FwdOnly), results in even greater improvements. This supports the validity and reasoning behind our theory. On the other hand, resetting only neurons with zero backward gradients (ReSiN-BwdOnly) performs the worst. A vanishing backward score may arise from saturation, disconnection, cancellation, or local insensitivity; some such neurons still have active forward activations, so resetting them destroys useful representations.

\subsection{Plug to Mixture of Experts}

To validate the effectiveness of ReSiN, we apply it to a MoE network architecture. The PA-MoE paper \cite{he2025plasticity} has already demonstrated that MoE possesses a certain ability to preserve plasticity. In this subsection, we investigate whether ReSiN can similarly achieve performance improvements when integrated into a MoE network. The MoE-based method equipped with ReSiN is referred to as ReSiN-MoE.

\begin{figure}[!htbp]
\centering
\includegraphics[width=2.0in]{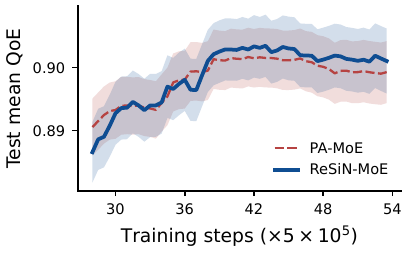}%
\caption{Comparison of test results during the training process.}
\label{ablation_training}
\end{figure}

To obtain conclusions with stronger statistical significance, we save the model every 500 training steps and then evaluate each saved checkpoint on the test set using three different random seeds. The resulting test curves are further smoothed using a moving-window filter. As shown in Fig.~\ref{ablation_training}, the model performance gradually improves as the number of training steps increases. ReSiN-MoE begins to outperform PA-MoE after approximately $4000 \times 500$ steps. Before that point, the two methods exhibit similar performance because only a small number of neurons have lost plasticity. However, as training progresses, an increasing number of neurons lose plasticity, and accurately identifying and resetting these neurons enables ReSiN-MoE to achieve a clear performance advantage.

\begin{figure}[!htbp]
\centering
\includegraphics[width=\columnwidth]{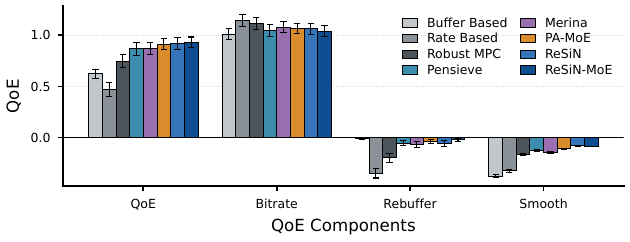}%
\caption{Performance comparison under different neural network.}
\label{ablation_moe_performance}
\end{figure}

We evaluate the final trained models on the test set. As shown in Fig.~\ref{ablation_moe_performance}, applying ReSiN directly to Pensieve achieves even better performance than PA-MoE. This indicates that, without modifying the network architecture, ReSiN alone can surpass the current state-of-the-art method, further highlighting the large potential for improvement in this direction. Moreover, applying ReSiN to the MoE architecture yields additional gains, demonstrating that ReSiN is an architecture-agnostic, stackable performance enhancement technique.

\section{Related Work}
\label{related_work}

We organize related work into three areas: handling non-stationary bandwidth variations (Section~\ref{NonstationaryEnvironment}), optimization methods in adaptive video streaming (Section~\ref{OptimizationMethod}), and plasticity-aware neural networks (Section~\ref{NeuralPlasticity}). Throughout, we highlight how ReSiN advances each area by preserving neural plasticity under non-stationary conditions.

\subsection{Non-stationary Environment}
\label{NonstationaryEnvironment}

Although Non-Stationary MDPs \cite{cheung2020reinforcement} and Continual Reinforcement Learning \cite{abel2023definition} have been defined and theoretically analyzed for some time, the challenge of non-stationary network environments poses fundamental difficulties for adaptive bitrate streaming systems. While Plume addresses the imbalanced distribution of network bandwidth data through clustering and prioritized trace sampling \cite{plume2024}, this approach fails to fully capture the non-stationary nature of the problem within its reinforcement learning framework. By modeling a non-stationary problem within a stationary Markov Decision Process (MDP) framework, it encounters inherent limitations in handling truly dynamic network conditions.  To address highly volatile network conditions, some approaches have turned to meta-reinforcement learning \cite{wang2024mmvs}. Meta-learning strives to develop systems that can "learn how to learn" across multiple tasks \cite{li2023metaabr}, enabling rapid adaptation to new environments. However, while meta-learning enhances model generalization, it still relies on an underlying assumption that training tasks share certain stationary properties in their distributions. This makes it suboptimal for truly non-stationary environments where such assumptions may not hold. Alternative approaches like Self-Play Framework have shown promise in handling dynamic objective functions, though their application to network bandwidth optimization remains limited \cite{huang2021zwei,li2024optimizing}.

A significant number of existing approaches combine bandwidth prediction with search algorithms to optimize video chunk selection \cite{lebreton2024long,wang2024enhancing,yin2024learning}. However, these methods' heavy reliance on prediction accuracy becomes particularly problematic in non-stationary network environments, where the underlying bandwidth distribution experiences temporal shifts. The violation of the independently and identically distributed (i.i.d.) data assumption in non-stationary conditions fundamentally challenges these prediction-based approaches. As network dynamics evolve, models trained on historical data struggle to capture emerging patterns, resulting in suboptimal chunk selection decisions that ultimately degrade streaming quality.

\subsection{Optimization Method}
\label{OptimizationMethod}

The challenge of optimizing video streaming systems has spawned diverse algorithmic approaches, each with distinct limitations in handling real-world complexity. Traditional optimization-based methods, while mathematically rigorous, require comprehensive knowledge of system dynamics and constraints \cite{spiteri2020bola}. Their computational complexity often renders them impractical for real-time decision-making scenarios. Similarly, dynamic programming approaches face inherent scalability limitations due to the curse-of-dimensionality, as their state space grows exponentially with system variables. These methods also depend heavily on accurate transition probability models between states, which are often difficult to obtain in practice. While Lyapunov control-based methods offer theoretical stability guarantees, their practical implementation is hindered by stringent assumptions about system dynamics and the complexity of optimizing control parameters.

Model Predictive Control (MPC) represents another significant approach, but its effectiveness is intrinsically tied to the accuracy of network bandwidth predictions \cite{lebreton2022adaptive,lin2024adaptive}. In non-stationary network environments, where bandwidth distributions experience temporal shifts, MPC faces substantial challenges. These out-of-distribution scenarios lead to significant deviations between predicted and actual bandwidth patterns, cascading into suboptimal control decisions that manifest as degraded QoE metrics, including increased latency and buffer underflows \cite{alomar2023causalsim,yin2015control}.

Recent innovations have attempted to address these limitations, such as the Two-Stage Deep Reinforcement Learning approach for 360-degree Video Streaming Scheduling, which incorporates an auxiliary agent for prediction error compensation \cite{bi2024two}. However, this solution introduces its own challenges: significant computational overhead and an inflexible error correction mechanism that struggles to accommodate varying degrees of prediction discrepancies. Furthermore, research has shown that sequential optimization of the two stages does not guarantee global optimality for the complete two-stage optimization problem \cite{he2025understanding}, highlighting the need for more sophisticated approaches that can handle the inherent complexity of video streaming optimization.

\subsection{Neural Plasticity}
\label{NeuralPlasticity}

Reinforcement learning systems, under both model-based \cite{fu2025knowledge} and model-free algorithms \cite{dohare2024loss}, demonstrate a marked tendency to lose plasticity when operating in non-stationary environments. Current research approaches this challenge through two distinct methodologies. The first focuses on preventive measures during the training process, employing specialized techniques such as distribution-regulating normalization methods \cite{lyle2024normalization} and sophisticated weight constraint mechanisms. These mechanisms serve dual purposes: controlling parameter magnitudes and maintaining proximity to initial distributions \cite{elsayed2024weight,kumar2023maintaining}.

The second methodology addresses networks that have already experienced plasticity loss through parameter resetting approaches \cite{dohare2024loss, juliani2024study, liu2024neuroplastic}. This category encompasses several innovative solutions, including Neuroplastic Expansion \cite{liu2024neuroplastic}, which enables incremental network growth through strategic neuron addition, and Continual Backpropagation \cite{dohare2024loss}, which employs systematic identification and reinitialization of underperforming neurons. Similarly, methods like ReDO \cite{sokar2023dormant}, GraMa \cite{liumeasure} and Plasticity Injection \cite{nikishin2024deep} maintain network adaptability through strategic neuron management during training. However, while these approaches offer practical solutions, they fall short of explaining the fundamental mechanisms driving plasticity loss \cite{LyleZNPPD23, gulcehre2022an, disentanglingcausesplasticityloss, lewandowski2024directionscurvatureexplanationloss, Hare_Tortoise}. To address this limitation, we propose a  \textbf{forward and backward based} metric for quantifying neuronal plasticity, providing a more nuanced and theoretically grounded framework for understanding network plasticity characteristics.

\section{Conclusions and Future Work}
\label{conclusion}

This paper presents a comprehensive investigation of neural plasticity in network resource adaptation in non-stationary environments. Our theoretical analysis explains the decline in performance under non-stationary conditions by attributing it to a loss of neural plasticity. To address this issue, we propose the use of Silent Neurons, a novel metric that takes into account both forward outputs and backward gradients, and develop a corresponding reset mechanism. Our experiments confirm the existence of plasticity loss and demonstrate the effectiveness of our solution. Interestingly, our findings also suggest that neurons with zero outputs but active gradients can actually improve network adaptability. Overall, our results showcase the superior performance of our approach in non-stationary network environments.

Looking forward, several promising directions emerge for future research. First, our framework could be extended to more complex scenarios where non-stationarity arises from various sources beyond network bandwidth, such as adversarial perturbations or user-induced variations. Second, while our Silent Neuron metric provides a more comprehensive measure of neural plasticity, investigating alternative or complementary indicators of plasticity could further enhance our understanding of neural network adaptation in dynamic environments. These extensions would contribute to developing more robust and adaptive systems for real-world applications. A promising future direction is to analyze how ABR internal states (e.g., buffer occupancy, bandwidth) induce plasticity loss, and to develop ABR-specific solutions accordingly.

\bibliographystyle{IEEEtran}
\bibliography{reference}

\begin{IEEEbiography}[{\includegraphics
[width=1in,height=1.25in,clip,
keepaspectratio]{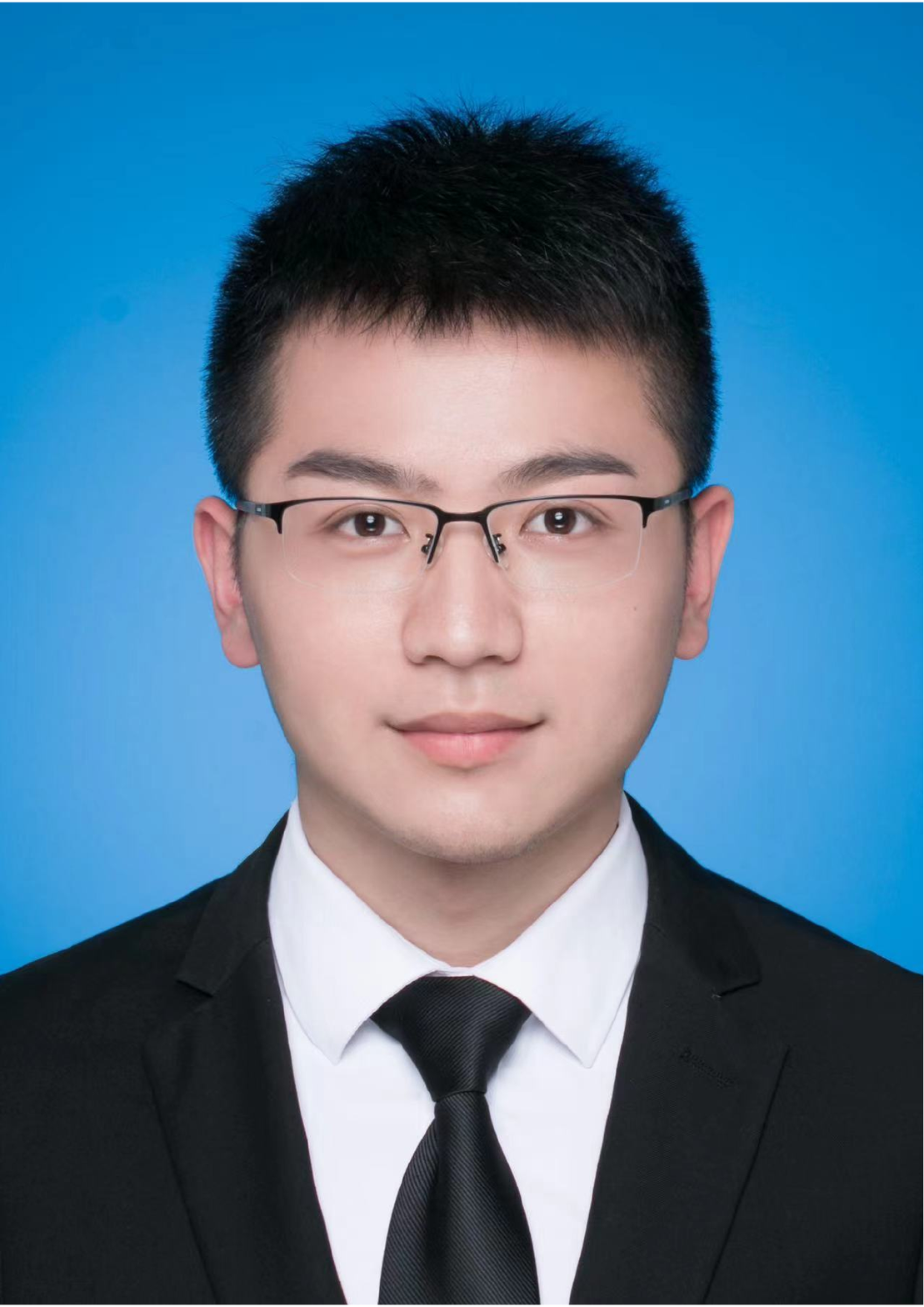}}]
{Zhiqiang He} is currently pursuing Ph.D. degree at the Graduate School of Informatics and Engineering, The University of Electro-Communications in Tokyo, Japan. He received MS degree in Control Science and Engineering from Northeastern University, Shenyang, China. He previously worked at Baidu and InspirAI, where he developed a master-level AI for the game Landlord that outperformed professional players. His research interests focus on deep reinforcement learning and its control applications. 
\end{IEEEbiography}


\begin{IEEEbiography}
[{\includegraphics[width=1in,height=1.25in,clip,keepaspectratio]{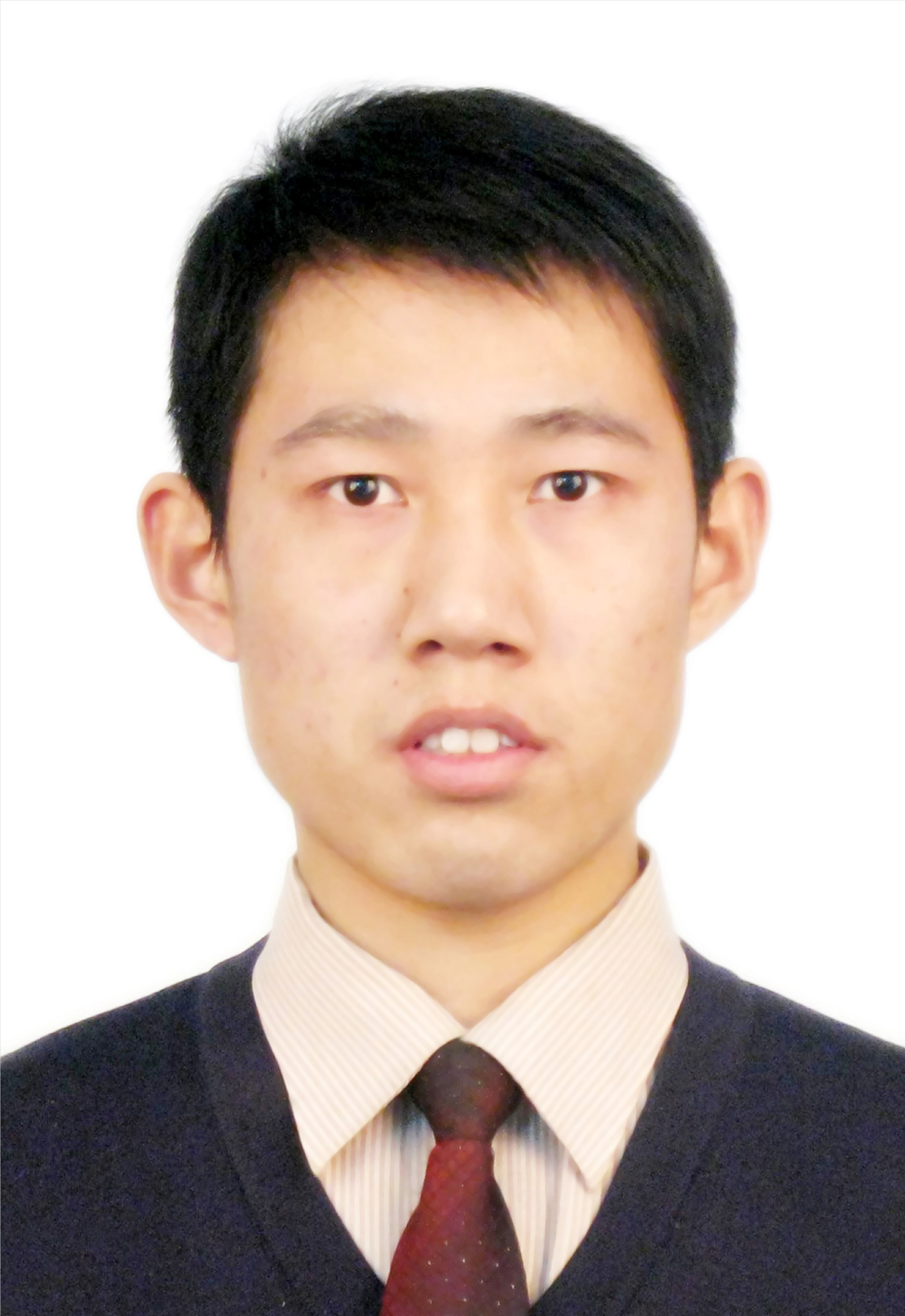}}]
{Zhi Liu} (S’11-M’14-SM’19) received the Ph.D. degree in informatics in National Institute of Informatics. He is currently an Associate Professor at the University of Electro-Communications. His research interest includes video network transmission and MEC. He is now an editorial board member of IEEE Transactions on Multimedia, IEEE Networks and Internet of Things Journal. He is a senior member of IEEE.
\end{IEEEbiography}


\newpage 
\clearpage

\appendix

\subsection{The definition of Overlap Coefficient for Neuron}

\begin{definition}[\textbf{Overlap Coefficient for Neuron}]
\label{OverlapCoefficientforNeuron} The overlap coefficient measures the similarity between the set of dormant or zero gradient neurons in each layer of a neural network at the current iteration and last iteration \cite{sokar2023dormant}. Let $\mathrm{A}$ represent the set of dormant neurons at the current iteration, and $\mathrm{B}$ represent the set of zero gradient neurons at the previous iteration. The overlap coefficient between $\mathrm{A}$ and $\mathrm{B}$ is defined as follows: $\operatorname{overlap}(\mathrm{A}, \mathrm{B})=\frac{|\mathrm{A} \cap \mathrm{B}|}{\min (|\mathrm{A}|,|\mathrm{B}|)}$. This metric is used to measure the proportion of neurons that remain consistent over time, providing insights into the network's insight state.
    
\end{definition}

\subsection{Bidirectional Dormancy Characterization Theorem Proof}
\label{BidirectionalDormancyCharacterizationTheoremProof}

Before presenting the proof of the Bidirectional Dormancy Characterization Theorem, we first establish several auxiliary Lemmas that will be used in the analysis. These results follow directly from Assumptions~\ref{continuity}–\ref{non_degeneracy} and provide the technical foundations required for the subsequent proof.

\begin{lemma}[\textbf{Forward Dormancy Lemma}]
    \label{ForwardDormancyLemma}
    If $s_{l,i} = 0$, then $h_{l,i}(\mathbf{x}) = 0$ for all $\mathbf{x} \in D$. 
\end{lemma}
\begin{proof}
    Since $s_{l,i} = 0$, by definition, $0=s_{l, i}=\frac{\mathbb{E}_{\mathbf{x} \in D}\left|h_{l, i}(\mathbf{x})\right|}{\frac{1}{H_l} \sum_{j=1}^{H_l} \mathbb{E}_{\mathbf{x} \in D}\left|h_{l, j}(\mathbf{x})\right|}$. 
    The non-degeneracy assumption guarantees that the denominator is bounded below by $m>0$, and thus, $\mathbb{E}_{\mathbf{x} \in D}\left|h_{l, i}(\mathbf{x})\right|=0$. Because $\left|h_{l, i}(\mathbf{x})\right| \geq 0$, a nonnegative function has zero expectation if and only if it is zero almost everywhere on $D$. Hence, $h_{l, i}(\mathbf{x})=0 \quad \text {for almost all } \mathbf{x} \in D$.
    By continuity of $h_{l, i}$ (Assumption~\ref{continuity}), a function that is zero almost everywhere must in fact be zero everywhere: if $h_{l, i}(\mathbf{x}_0)\neq 0$ at some point $\mathbf{x}_0\in D$, continuity would imply the existence of a neighborhood on which $h_{l, i}$ remains nonzero, contradicting the almost-everywhere zero property.
    Therefore, $h_{l,i}(\mathbf{x})=0$ for all $\mathbf{x}\in D$.
\end{proof}

\begin{lemma}[\textbf{Forward-to-Backward Dormancy Lemma}]
\label{forward_to_backward_lemma}
If $s_{l,i}=0$, then $\nabla h_{l,i}(\mathbf{x}) = 0$ for all $\mathbf{x} \in D$.
\end{lemma}

\begin{proof}
By Lemma~\ref{ForwardDormancyLemma}, the condition $s_{l,i}=0$ implies that $h_{l,i}(\mathbf{x}) = 0 \quad \text{for all } \mathbf{x}\in D$. Fix any $\mathbf{x}\in D$ and consider an arbitrary perturbation $\delta_{\mathbf{x}}$ such that the entire segment $\{\mathbf{x}+t\delta_{\mathbf{x}} : t\in[0,1]\} \subseteq D$. By the Mean Value Theorem for vector-valued functions \cite{luenberger1984linear}, there exists $t\in[0,1]$ satisfying $h_{l,i}(\mathbf{x}+\delta_{\mathbf{x}}) - h_{l,i}(\mathbf{x})
= \nabla h_{l,i}(\mathbf{x}+t\delta_{\mathbf{x}})^{T}\delta_{\mathbf{x}}$. Both terms on the left vanish because $h_{l, i}$ is identically zero on $D$. Thus, $\nabla h_{l,i}(\mathbf{x}+t\delta_{\mathbf{x}})^{T}\delta_{\mathbf{x}} = 0 \quad \text{for all admissible }\delta_{\mathbf{x}}$.

Suppose, for contradiction, that there exists a point $\mathbf{y}\in D$ such that 
$\nabla h_{l,i}(\mathbf{y})\neq 0$.  
Choose $\delta_{\mathbf{x}}=\epsilon \nabla h_{l,i}(\mathbf{y})$ with $\epsilon>0$ sufficiently small so that $\mathbf{y}+t\epsilon \nabla h_{l,i}(\mathbf{y}) \in D$ for all $t\in[0,1]$.  
Then, $\nabla h_{l,i}(\mathbf{y}+t\epsilon \nabla h_{l,i}(\mathbf{y}))^{T}
\bigl(\epsilon\nabla h_{l,i}(\mathbf{y})\bigr)=0$.
As $\epsilon\to 0$, Assumption~\ref{continuity} ensures $\nabla h_{l,i}(\mathbf{y}+t\epsilon\nabla h_{l,i}(\mathbf{y}))
\longrightarrow \nabla h_{l,i}(\mathbf{y})$, which is nonzero by assumption. Hence the inner product above becomes $\epsilon \|\nabla h_{l,i}(\mathbf{y})\|^{2} \neq 0 \quad \text{for sufficiently small }\epsilon>0$, contradicting the fact that 
$\nabla h_{l,i}(\mathbf{x}+t\delta_{\mathbf{x}})^{T}\delta_{\mathbf{x}} = 0$. Therefore, no such point $\mathbf{y}$ can exist, and we conclude $\nabla h_{l,i}(\mathbf{x}) = 0 \quad \text{for all } \mathbf{x}\in D$.
\end{proof}

\begin{lemma}[\textbf{Backward Dormancy Lemma}]
\label{backward_dormancy_lemma}
Suppose the gradient $\nabla h_{l,i}(\mathbf{x}) = 0$ for all $\mathbf{x} \in \mathcal{X}_l$. Then $h_{l,i}(\mathbf{x})$ is constant on $\mathcal{X}_l$. Moreover, if $h_{l,i}(\mathbf{x}_0)=0$ for some $\mathbf{x}_0 \in \mathcal{X}_l$, then $h_{l,i}(\mathbf{x}) = 0$ for all $\mathbf{x}\in \mathcal{X}_l$, and thus the neuron is dormant $(s_{l,i}=0)$.
\end{lemma}

\begin{proof}
Since $\nabla h_{l,i}(\mathbf{x}) = 0$ for all $\mathbf{x}\in \mathcal{X}_l$, the function cannot vary with respect to its input. By standard multivariable calculus, a continuously differentiable function with vanishing gradient on a connected domain must be constant. More explicitly, for any $\mathbf{x}_1, \mathbf{x}_2 \in \mathcal{X}_l$, $h_{l,i}(\mathbf{x}_2) - h_{l,i}(\mathbf{x}_1)
= \int_{\mathbf{x}_1}^{\mathbf{x}_2} \nabla h_{l,i}(\mathbf{u})\, d\mathbf{u}
= 0$, which implies $h_{l,i}(\mathbf{x}_2) = h_{l,i}(\mathbf{x}_1)$. Therefore, $h_{l,i}$ is constant on $\mathcal{X}_l$. Denote this constant by \(c\), i.e., $h_{l,i}(\mathbf{x}) = c \quad \forall \mathbf{x}\in \mathcal{X}_l$.

If there exists $\mathbf{x}_0\in \mathcal{X}_l$ such that $h_{l,i}(\mathbf{x}_0)=0$, then the constant must satisfy $c=0$, implying $h_{l,i}(\mathbf{x}) = 0 \quad \forall \mathbf{x}\in \mathcal{X}_l$.
In particular, this holds for all $\mathbf{x}\in D \subseteq \mathcal{X}_l$. Hence,
\[
\mathbb{E}_{\mathbf{x}\in D}|h_{l,i}(\mathbf{x})| \! = \! 0 \! \Longrightarrow \!
s_{l,i} \! = \! \frac{0}{
\frac{1}{H_l} \sum_{k=1}^{H_l} 
\mathbb{E}_{\mathbf{x}\in D}|h_{l,k}(\mathbf{x})|
} \!= \! 0.
\]
Thus, the neuron is fully dormant.

\textbf{Remark.} If $c \neq 0$, then $\mathbb{E}_{\mathbf{x} \in D}\left|h_{l, i}(\mathbf{x})\right|$ would be $|c| \neq 0$, so $s_{l, i}$ could potentially be nonzero, indicating the neuron is not dormant. Therefore, the key condition that $h_{l,i}(\mathbf{x})=0$ at some point $\mathbf{x}_{0}$ in $\mathbf{X}_{l}$ (thus forcing $c=0$) is crucial to concluding $s_{l,i}=0$. In many network architectures, one can ensure $h_{l, i}(0)=0$ by design (e.g., bias initialized to zero and activation is ReLU), or rely on data/architectural constraints that force a zero constant rather than a nonzero one.
\end{proof}

These lemmas collectively unveil the intrinsic properties of dormant neurons from complementary perspectives. Building upon these foundational results, we present a comprehensive characterization of bidirectional dormancy in Theorem \ref{BidirectionalDormancyCharacterization}, which provides a unified framework for analyzing neural network plasticity.

\textit{Theorem} \ref{BidirectionalDormancyCharacterization} \textit{\textbf{Bidirectional Dormancy Characterization}}: Let $D \subseteq \mathcal{X}_l \subseteq \mathbb{R}^{k_l}$ be the domain from which inputs $\mathbf{x}$ are drawn. Consider layer $l$ with $H_l$ neurons, and let the activation of the $i$-th neuron be $h_{l,i}(\mathbf{x}) = \sigma_l(\mathbf{w}_{l,i}^{T} \mathbf{x} + b_{l,i})$, where $\sigma_l$ is continuously differentiable and applied elementwise. The neuron’s \emph{dormancy index} is defined by
\[
s_{l,i}
=
\frac{
\mathbb{E}_{\mathbf{x}\in D}\bigl|h_{l,i}(\mathbf{x})\bigr|
}{
\frac{1}{H_l}\sum_{j=1}^{H_l} \mathbb{E}_{\mathbf{x}\in D}\bigl|h_{l,j}(\mathbf{x})\bigr|
}.
\]
Under Assumptions~\ref{continuity}–\ref{non_degeneracy}, the following statements are equivalent:
\begin{itemize}
    \item[(A)] \textbf{Dormancy:} $s_{l,i}=0$, equivalently $\mathbb{E}_{\mathbf{x}\in D}|h_{l,i}(\mathbf{x})|=0$.
    \item[(B)] \textbf{Zero gradient on $D$ and one zero activation:}
    \[
    \nabla h_{l,i}(\mathbf{x}) = 0 \quad \forall\, \mathbf{x}\in D,
    \qquad
    \exists\, \mathbf{x}_0 \in D : h_{l,i}(\mathbf{x}_0)=0.
    \]
\end{itemize}

\begin{proof}
\textbf{(A $\Rightarrow$ B).}
If $s_{l,i}=0$, Lemma~\ref{ForwardDormancyLemma} implies that $h_{l,i}(\mathbf{x}) = 0 \quad \forall\, \mathbf{x}\in D$. Applying Lemma~\ref{forward_to_backward_lemma} then yields $\nabla h_{l,i}(\mathbf{x}) = 0 \quad \forall\, \mathbf{x}\in D$. Since the neuron outputs zero everywhere on $D$, any point $\mathbf{x}_0\in D$ satisfies $h_{l,i}(\mathbf{x}_0)=0$, establishing (B).

\medskip
\textbf{(B $\Rightarrow$ A).}
Assume (B) holds. Because $\nabla h_{l,i}(\mathbf{x})=0$ for all $\mathbf{x}\in D$, the function $h_{l,i}$ must be constant on $D$. Let this constant be $c$. Since $h_{l,i}(\mathbf{x}_0)=0$ for some $\mathbf{x}_0\in D$, we have $c=0$, hence $h_{l,i}(\mathbf{x})=0 \quad \forall\, \mathbf{x}\in D$. Applying Lemma~\ref{ForwardDormancyLemma} in reverse, the identity \(h_{l,i}\equiv 0\) on $D$ implies $s_{l,i}=0$. Thus (A) follows.
\end{proof}

\subsection{Silent Neuron Characterization Theorem Proof}
\label{SilentNeuronCharacterizationTheoremProof}

\textit{Theorem} \ref{SilentNeuronCharacterization} \textit{\textbf{(Silent Neuron Characterization)}}. 
Let $D \subseteq \mathcal{X}_l \subseteq \mathbb{R}^{k_l}$ be the domain from which inputs $\mathbf{x}$ are drawn. Consider a layer $l$ with $H_l$ neurons, and let the activation of the $i$-th neuron be $h_{l,i}(\mathbf{x}) = \sigma_l(\mathbf{w}_{l,i}^{T} \mathbf{x} + b_{l,i})$, where $\sigma_l$ is continuously differentiable and applied elementwise. Define the neuron's \emph{activity index} as
\[
\xi_{l,i}
=
\frac{
\mathbb{E}_{\mathbf{x}\in D}|h_{l,i}(\mathbf{x})|}{
\frac{1}{H_l}\sum_{j=1}^{H_l}\mathbb{E}_{\mathbf{x}\in D}|h_{l,j}(\mathbf{x})|
} + \mathbb{E}_{\mathbf{x}\in D}|g_{l,i}(\mathbf{x})|,
\]
where $g_{l, i}(\mathbf{x})$ is defined in Eq. (4).  Under Assumptions~\ref{continuity}–\ref{non_degeneracy}, where $m>0$ denotes the non-degeneracy lower bound from Assumption 3, and suppose there exist constants $M_h > 0$. Then, the following statements are equivalent:
\begin{itemize}
    \item[(A)] \textbf{Silence:} For any $\epsilon > 0$, the activity index satisfies $\xi_{l,i} < \epsilon$.
    \item[(B)] \textbf{Vanishing forward and backward activity (on $D$):} For any $\delta_{d} > 0$ and $\delta_{g} > 0$, we have
    \[
    \mathbb{E}_{\mathbf{x}\in D}|h_{l,i}(\mathbf{x})| < \delta_{d}
    \quad \text{and} \quad
    \mathbb{E}_{\mathbf{x}\in D}|g_{l,i}(\mathbf{x})| < \delta_{g}.
    \]
\end{itemize}

\begin{proof}
\textbf{(A $\Rightarrow$ B).}  
Suppose $\xi_{l,i} < \epsilon$. Since both terms are non-negative, then:
\begin{equation}
\frac{
\mathbb{E}_{\mathbf{x}\in D}|h_{l,i}(\mathbf{x})|
}{
\frac{1}{H_l}\sum_{j=1}^{H_l}\mathbb{E}_{\mathbf{x}\in D}|h_{l,j}(\mathbf{x})|
} < \epsilon \quad \text{and} \quad \mathbb{E}_{\mathbf{x}\in D}|g_{l,i}(\mathbf{x})| < \epsilon 
\end{equation}
For the activation term, since $\frac{1}{H_l}\sum_{j=1}^{H_l}\mathbb{E}_{\mathbf{x}\in D}|h_{l,j}(\mathbf{x})| \leq M_{h}$, we have $\mathbb{E}_{\mathbf{x}\in D}|h_{l,i}(\mathbf{x})| < M_{h} \epsilon$. Since (A) holds for any $\epsilon > 0$. For arbitrary $\delta_{d} > 0$ and $\delta_{g} > 0$, let us consider $\epsilon = \min( \frac{\delta_{d}}{M_{h}}, \delta_{g})$. Then:
\[
\mathbb{E}_{\mathbf{x}\in D}|h_{l,i}(\mathbf{x})| < M_h\epsilon \leq \delta_{d},
\qquad
\mathbb{E}_{\mathbf{x}\in D}|g_{l,i}(\mathbf{x})| < \epsilon \leq \delta_{g},
\]
which implies (B) for arbitrarily small $\delta_{d}$ and $\delta_{g}$.

\medskip
\textbf{(B $\Rightarrow$ A).}  
Assume (B) holds. We can choose a specific $\delta_{d} = \frac{m \epsilon}{2}$ and $\delta_{g} = \frac{\epsilon}{2}$. Then
\[
\xi_{l,i}
=
\frac{
\mathbb{E}_{\mathbf{x}\in D}|h_{l,i}(\mathbf{x})|
}{
\frac{1}{H_l}\sum_{j=1}^{H_l}\mathbb{E}_{\mathbf{x}\in D}|h_{l,j}(\mathbf{x})|
} +
\mathbb{E}_{\mathbf{x}\in D}|g_{l,i}(\mathbf{x})|
< \frac{\delta_{d}}{m} + \delta_{g}
= \epsilon
\]
This proves $\xi_{l,i} < \epsilon$, establishing (A).
\end{proof}

\subsection{The Silent Neuron–Enhanced PPO Performance Bound}
\label{TheSilentNeuronEnhancedPPOPerformanceBound}

Before deriving the Silent Neuron–Enhanced PPO performance bound, we first establish a preliminary lemma to streamline the subsequent proof based on the PA-MoE theory.

\begin{lemma}[Noise energy in the silent subspace]
\label{lemma:silent_noise_energy}
Let $E_t = \eta\gamma \Pi_t \boldsymbol{\epsilon}_t$ with $\boldsymbol{\epsilon}_t \sim \mathcal{N}(0, I_d)$ and $\Pi_t$ an orthogonal projection of rank $d_{s,t}$. Then $\mathbb{E}\,\|E_t\|_2^2 = \eta^2 \gamma^2 \operatorname{tr}(\Pi_t) = \eta^2 \gamma^2 d_{s,t} \le \eta^2 \gamma^2 d_s$.
\end{lemma}

\begin{proof}
Since $\boldsymbol{\epsilon}_t \sim \mathcal{N}(0, I_d)$, we have $\mathbb{E}\,\|E_t\|_2^2
\! = \! \eta^2\gamma^2 \mathbb{E}\bigl[\! \boldsymbol{\epsilon}_t^{T} \! \Pi_t^{T} \! \Pi_t \boldsymbol{\epsilon}_t \!\bigr] \! = \! \eta^2\gamma^2 \mathbb{E}\bigl[\! \boldsymbol{\epsilon}_t^{T} \Pi_t \boldsymbol{\epsilon}_t \!\bigr] \! = \! \eta^2 \gamma^2 \! \operatorname{tr}(\Pi_t)$, where we used $\Pi_t^2=\Pi_t$ and the standard identity
$\mathbb{E}[\boldsymbol{\epsilon}^{T} A \boldsymbol{\epsilon}] = \operatorname{tr}(A)$ for $\boldsymbol{\epsilon}\sim\mathcal{N}(0,I_d)$. 
Since $\operatorname{tr}(\Pi_t) = d_{s,t}$ and $d_{s,t}\le d_s$ by definition, the result follows.
\end{proof}

Lemma~\ref{lemma:silent_noise_energy} shows that the effective noise energy depends only on the dimension of the silent neuron subspace, rather than the full parameter dimension $d$. This allows us to inject relatively strong perturbations into under-utilized neurons without destabilizing the overall optimization dynamics.

Let $L_t(\mathbf{w})$, with $\mathbf{w} \in \mathbb{R}^d$, denote the ppo objective at time $t$. Within this restricted proximity, the global objective—which is generally non-convex over the full parameter space—admits a local second-order approximation. Building on the analytical frameworks of \cite{agarwal2021theory,schulman2015trust}, we adopt standard regularity assumptions on the local behavior of the objective, namely L-smoothness and (local) $\mu$-strong convexity within the trust region. Notably, the Critic is trained with a mean squared error loss, which is a convex objective. While these assumptions do not hold globally for neural-network-parameterized policies, they are commonly invoked in the analysis of trust-region methods to characterize convergence rates and parameter adaptation dynamics under non-stationary conditions. We note that this analysis is only within the trust region; it provides no guarantee once updates move outside this region. Following the theoretical assumptions in PA-MoE~\cite{he2025plasticity}, we adopt:

\textbf{(A4) L-Smoothness.} There exists $L>0$ such that, for all $\mathbf{w},\mathbf{w}'\in\mathbb{R}^d$, $\|\nabla L_t(\mathbf{w})-\nabla L_t(\mathbf{w}')\| \le L\|\mathbf{w}-\mathbf{w}'\|$.

\textbf{(A5) $\mu$-Strong Convexity.} There exists $\mu>0$ such that, for all $\mathbf{w},\mathbf{w}'\in\mathbb{R}^d$, $L_t(\mathbf{w}') \ge L_t(\mathbf{w}) + \nabla L_t(\mathbf{w})^{\top}(\mathbf{w}'-\mathbf{w}) + \frac{\mu}{2}\|\mathbf{w}'-\mathbf{w}\|^2$.

\textbf{(A6) Nonstationarity.} Let the time-varying optimum be $\mathbf{w}_t^* \in \arg\min_{\mathbf{w}} L_t(\mathbf{w})$, and define the path length $P_T := \sum_{t=1}^{T-1}\|\mathbf{w}_{t+1}^*-\mathbf{w}_t^*\|$.
We assume $P_T < \infty$.

\textbf{Theorem \ref{thm:resin-error-bound}:}  (\textbf{Tracking Analysis of Subspace-Restricted Updates for PPO}) 
Under Assumptions (A4)--(A6) with step-size $0 < \eta \leq 1/L$, and with the update rule, $\boldsymbol{w}_{t+1} = \boldsymbol{w}_t - \eta\, \nabla L_t\bigl(\boldsymbol{w}_t\bigr) + \eta\, \gamma\, \Pi_{t}\epsilon_t, \quad \epsilon_t \sim \mathcal{N}\left(0, I_d\right)$, the average squared error satisfies, $\frac{1}{T} \sum_{t=1}^{T} \mathbb{E}\|e_t\|^2 \le \frac{2}{\mu \eta T} \left(\mathbb{E}||e_{0}||^{2} + \frac{2 P_{T}^{2}}{\mu \eta} \right) + \frac{2 \eta \gamma^{2}}{\mu} \cdot \frac{1}{T}\sum_{t=0}^{T-1}d_{s,t}$, where the error is defined as $e_t = \boldsymbol{w}_t - \boldsymbol{w}_t^*$, the path length of the optimal parameters is $P_T = \sum_{t=1}^{T-1}\|\boldsymbol{w}_{t+1}^* - \boldsymbol{w}_t^*\|$.
\begin{proof}
Let the movement of the optimal parameters be represented by $\Delta_t := \boldsymbol{w}_{t+1}^* - \boldsymbol{w}_t^*$. The tracking error evolves as
\begin{align}
& e_{t+1} = \boldsymbol{w}_{t+1} - \boldsymbol{w}_{t+1}^* \nonumber\\[1mm]
&= \boldsymbol{w}_t - \eta\, \nabla L_t\bigl(\boldsymbol{w}_t\bigr) + \eta\, \gamma\, \Pi_{t} \epsilon_t - \boldsymbol{w}_{t+1}^* \nonumber\\[1mm]
&= \underbrace{\left[\boldsymbol{w}_t - \boldsymbol{w}_{t}^* - \eta\Bigl(\nabla L_t\bigl(\boldsymbol{w}_t\bigr) - \nabla L_t\bigl(\boldsymbol{w}_t^*\bigr)\Bigr)\right]}_{=: A_t} - \Delta_t + \underbrace{\eta\, \gamma\, \Pi_{t} \epsilon_t}_{=: E_t} \nonumber \\
& =\underbrace{\left(e_t-\eta\left(\nabla L_t\left(\boldsymbol{w}_i^t\right)-\nabla L_t\left(\boldsymbol{w}_t^*\right)\right)\right)}_{\text {Contraction Term } A_t}-\underbrace{\Delta_t}_{\text {Drift Term }}+\underbrace{\eta \gamma \Pi_{t} \epsilon_t}_{\text {Noise Term }}. \nonumber
\end{align}
Thus, $e_{t+1} = A_t - \Delta_t + E_t$. From PA-MoE \textbf{Contraction Property of \(A_t\)}, it holds that $\|A_t\|^2 \le (1 - \mu \eta)\,\|e_t\|^2$ with assumption $\eta \leq 1/L$.

\textbf{Recursive bound with noise and drift.}  
Expanding the square and taking expectation ($\mathbb{E} \left[ \epsilon_{t} \right] = 0$)

\begin{align}
    \left\|e_{t+1}\right\|^2 & \leq \nonumber \\ 
    & \left\|A_t-\Delta_t\right\|^2+2 \eta \gamma \Pi_{t} \left\langle A_t-\Delta_t, \epsilon_t\right\rangle+\eta^2 \gamma^2d_{s,t}\left\|\epsilon_t\right\|^2. \nonumber
\end{align}
Using the parameterized inequality \cite{he2025plasticity} $\|a+b\|^2 \! \leq \! (1 \! + \!\alpha)\|a\|^2 \! + \!\left(1 \!+\! \alpha^{-1}\right)\|b\|^2, a > 0, \nonumber$ gives $\mathbb{E}\left\|e_{t+1}\right\|^2 \! \leq \! (1+\alpha)(1-\mu \eta) \mathbb{E}\left\|e_t\right\|^2+\left(1 \! + \! \alpha^{-1}\right)\left\|\Delta_t\right\|^2 \! + \! \eta^2 \gamma^2 d_{s,t}$. 
Choose $\alpha=\mu \eta / 2$ ensuring $(1 + \alpha)(1 - \mu \eta) \leq 1 - \mu \eta / 2$. This yield:
\begin{equation}
    \mathbb{E}\left\|e_{t+1}\right\|^2 \leq\left(1-\frac{\mu \eta}{2}\right) \mathbb{E}\left\|e_t\right\|^2+\frac{2}{\mu \eta}\left\|\Delta_t\right\|^2+\eta^2 \gamma^2 d_{s,t}. \nonumber
\end{equation}

\textbf{Telescoping Sum and Final Bound.} 
Iterating the inequality gives 
\begin{align}
    \mathbb{E}\left\|e_t\right\|^2 \! \leq\left(1-\frac{\mu \eta}{2}\right)^t \mathbb{E}&\left\|e_0\right\|^2 \! +\! \frac{2}{\mu \eta} \! \sum_{k=0}^{t-1}\left(1-\frac{\mu \eta}{2}\right)^{t-1-k}\left\|\Delta_k\right\|^2 \nonumber \\
    &+ \eta^2 \gamma^2 \sum_{k=0}^{t-1}\left(1-\frac{\mu \eta}{2}\right)^{t-1-k} \cdot d_{s,k}. \nonumber
\end{align}
Summing and averaging the error from $t=1$ to $T$. The initial factor decays geometrically: $\frac{1}{T} \sum_{t=1}^T\left(1-\frac{\mu \eta}{2}\right)^t \mathbb{E}\left\|e_0\right\|^2 \leq \frac{1}{T} \cdot \frac{2}{\mu \eta} \mathbb{E}\left\|e_0\right\|^2$.

For the drift term, interchange summation order and utilize the geometric series formula: $\frac{1}{T} \! \sum_{t=1}^T \! \frac{2}{\mu \eta} \! \sum_{k=0}^{t-1} \! \left(1 \! - \! \frac{\mu \eta}{2} \! \right)^{t \! - \! 1 \! - \!k} \!\left\|\!\Delta_k\!\right\|^2 \! \leq \! \frac{1}{T}\! \cdot\! \frac{4}{\mu^2 \eta^2} \! \sum_{k=0}^{T-1} \! \left\|\!\Delta_k\!\right\|^2$.

The noise contribution at each time step follows a geometric series: 
$$\frac{1}{T} \sum_{t=1}^{T} \eta^{2} \gamma^{2} \sum_{k=0}^{t-1} (1 -\frac{\mu \eta}{2})^{t-1-k} d_{s,k} \leq \frac{2 \eta \gamma^{2}}{\mu} \cdot \frac{1}{T}\sum_{k=0}^{T-1}d_{s,k}.$$
After a refined constant adjustment, we arrive at
\begin{align}
    \frac{1}{T} \sum_{t=1}^{T} \mathbb{E}\|e_t\|^2 & \leq \frac{1}{T}\cdot \frac{2}{\mu\eta} \mathbb{E}||e_{0}||^{2} + \frac{4P_{T}^{2}}{ \mu^{2} \eta^{2} T} + \frac{2 \eta \gamma^{2}}{\mu} \cdot \frac{1}{T} \sum_{t=0}^{T-1}d_{s,t} \nonumber \\
    & = \frac{2}{\mu \eta T}\left( \mathbb{E}||e_{0}||^{2} + \frac{2 P_{T}^{2}}{\mu \eta}\right) + \frac{2 \eta \gamma^{2}}{\mu} \cdot \frac{1}{T} \sum_{t=0}^{T-1}d_{s,t} \nonumber \\
    & \leq \frac{2}{\mu \eta T}\left( \mathbb{E}||e_{0}||^{2} + \frac{2 P_{T}^{2}}{\mu \eta}\right) + \frac{2 \eta \gamma^{2}}{\mu} \cdot d. \nonumber
\end{align}
The expression $\frac{2}{\mu \eta T}\left( \mathbb{E}\|e_{0}\|^{2} + \frac{2 P_{T}^{2}}{\mu \eta}\right) + \frac{2 \eta \gamma^{2}}{\mu} \cdot d$ represents the bound of PA-MoE \cite{he2025plasticity}.
\end{proof}

\textbf{Remark (Strict Improvement):} ReSiN's bound is strictly tighter than PA-MoE's whenever $\frac{1}{T} \sum_{t=0}^{T-1} d_{s,t} < d$. Since $d_{s, t} \leq d$ holds at every step by construction (the silent subspace is a subspace of the full parameter space), a sufficient condition for strict inequality is the existence of at least one step  $t^{*} \in \{0, \cdots, T-1\}$ with $d_{s, t^{*}} < d$. In this case, $\frac{1}{T}\sum_{t=0}^{T-1} d_{s,t} \leq \frac{(T-1)d + d_{s,t^*}}{T} < d$, and the strict inequality follows. Fig.~\ref{dormantandzerogradient} empirically confirms that this condition holds throughout ABR training under non-stationary conditions, as the silent subspace dimension $d_{s, k}$ is consistently smaller than $d$ at every observed training step.

\subsection{Local Curvature Analysis around the PPO Iterates}
\label{EmpiricalJustificationAssumption}

Theorem 3 is stated as a conditional result under $\mu$-strong convexity assumption (A5). We do not claim that this assumption holds for the non-convex PPO objective; strong convexity is adopted as a standard idealization for tracking analysis. To contextualize the assumption, we instead report the local curvature of the loss around the training iterates.

\begin{figure}[!htbp]
\centering
\includegraphics[width=\columnwidth]{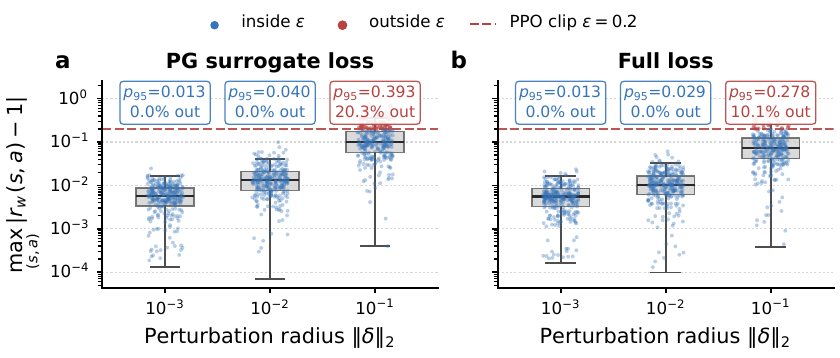}%
\caption{Distribution of perturbation points across different perturbation radii.}
\label{theory_prove_radius_radii}
\end{figure}

We sample perturbation points within the PPO trust region and examine the local Hessian spectrum along training (Fig.~\ref{theory_prove_radius_radii}, Fig.~\ref{fig:hessian_spectrum}). The smallest Hessian eigenvalue $\lambda_{min}(H)$ remains slightly negative, confirming that the objective is locally non-convex. At the same time, the spectrum is dominated by a few large positive eigenvalues, so the curvature is positive along the principal directions of the loss.

\begin{figure}[!htbp]
\centering
\includegraphics[width=\columnwidth]{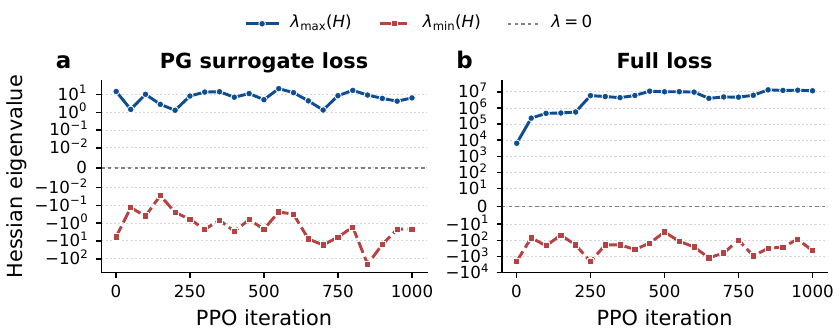}%
\caption{Hessian spectrum over PPO training.}
\label{fig:hessian_spectrum}
\end{figure}

We report these measurements as a description of the local geometry. Extending the analysis to weaker local conditions consistent with the measured curvature is left to future work.

\subsection{Ablation Results under an Additional Random Seed}
\label{AblationSeed1}

\begin{figure}[!htbp]  
  \centering
  \includegraphics[width=\columnwidth]{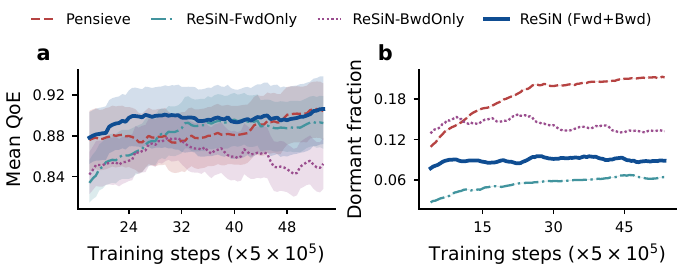}
  \caption{Performance gains and dormancy metrics under real-world trace training. The model is trained with seed 1 and evaluated on the test trace.}
  \label{ppo_reward_ablation_comparison_seed1}
\end{figure}

To verify that our ablation findings are not specific to a single random seed, we repeat the experiment with an additional seed (seed 1). Averaged over the full trajectory, the mean QoE is $0.887 \pm 0.050$ (Pensieve), $0.873 \pm 0.046$ (ReSiN-FwdOnly), $0.854 \pm 0.045$ (ReSiN-BwdOnly) and $0.893 \pm 0.051$ (ReSiN Fwd+Bwd), with corresponding dormant fractions of $0.184 \pm 0.001$, $0.055 \pm 0.001$, $0.141 \pm 0.001$ and $0.089 \pm 0.001$ (mean $\pm$ 95\% CI). These results are consistent with the seed-0 results reported in the main text: ReSiN (Fwd+Bwd) achieves the highest QoE, while ReSiN-BwdOnly performs worst, confirming that the trends hold across random seeds.

\end{document}